\documentclass{article} 
 
\usepackage[nonatbib, final]{neurips_2020}

\usepackage[utf8]{inputenc} 
\usepackage[T1]{fontenc}    
\usepackage{hyperref}       
\usepackage{url}            
\usepackage{booktabs}       
\usepackage{amsfonts}       
\usepackage{nicefrac}       
\usepackage{microtype,soul}      

\usepackage{color}
\usepackage{graphicx}

\usepackage{comment}

\usepackage{kz_style}

\DeclareMathOperator*{\Rd}{\mathbb{R}^d}

\DeclareMathOperator*{\E}{\mathbb{E}}

\newcommand*\bigcdot{\mathpalette\bigcdot@{.5}}
\renewcommand{\footnotesize}{\scriptsize}

\title{      
An Improved Analysis of 
(Variance-Reduced) Policy Gradient and Natural Policy Gradient Methods} 

\author{%
  Yanli Liu\footnotemark[1] \quad Kaiqing Zhang\footnotemark[2] \quad Tamer Ba\c{s}ar\footnotemark[2] \quad Wotao Yin\footnotemark[1]\\
  $^\natural$Department of Mathematics, University of California, Los Angeles\\
  $^\dagger$Department of ECE and CSL, University of Illinois at Urbana-Champaign\\
\texttt{\{yanli, wotaoyin\}@math.ucla.edu, \{kzhang66, basar1\}@illinois.edu} 
}

\begin{document}

\maketitle

\begin{abstract}

In this paper, we revisit and improve the convergence of policy gradient (PG), natural PG (NPG) methods, and their variance-reduced variants, under general smooth policy parametrizations. More specifically, with the Fisher information matrix of the policy being positive definite: i) we show that a state-of-the-art variance-reduced PG method, which has only been shown to converge to stationary points, converges to the globally optimal value up to some inherent function approximation error due to policy parametrization; ii) we show that NPG enjoys a lower sample complexity; iii) we propose SRVR-NPG, which incorporates variance-reduction into the NPG update. Our improvements follow from an observation that the convergence of (variance-reduced) PG and NPG methods can improve each other: the stationary convergence analysis of PG can be applied to NPG as well, and the global convergence analysis of NPG can help to establish the global convergence of (variance-reduced) PG methods. Our analysis carefully integrates the advantages of these two lines of works. 
    Thanks to this improvement, we have also made variance-reduction for NPG possible, with both global convergence and an efficient finite-sample complexity.



\end{abstract}

\section{Introduction}\label{sec:intro}
 
Policy gradient (PG) methods, or more generally direct policy search methods,  have long been recognized as one of the foundations  of reinforcement learning (RL) \cite{sutton2018reinforcement}. Specifically, PG methods directly search for the optimal policy parameter that maximizes the long-term return in Markov decision processes  (MDPs), following  the policy gradient ascent direction \cite{williams1992simple,sutton2000policy}. This search direction can be more efficient using a preconditioning matrix, e.g., using the natural PG direction \cite{kakade2002natural}. These methods have achieved tremendous empirical successes recently, especially boosted by the power of (deep) neural networks for  policy parametrization \cite{schulman2015trust,lillicrap2015continuous,mnih2016asynchronous,schulman2017proximal}. These successes are primarily attributed to the fact  that PG methods naturally incorporate \emph{function approximation} for policy parametrization, in order to handle massive and even continuous state-action spaces.

In practice, the policy gradients are usually estimated via samples using Monte-Carlo rollouts and bootstrapping \cite{williams1992simple,baxter2001infinite}. Such stochastic PG methods notoriously suffer from very high variances, which not only destabilize but also slow down the convergence. Several conventional approaches have been advocated to reduce the variance  of PG methods, e.g., by adding a baseline \cite{sutton2000policy,wu2018variance}, or by using function approximation for estimating the value function, namely, developing  actor-critic algorithms \cite{konda2000actor,peters2008natural,bhatnagar2009natural}. More recently, motivated by the advances of variance-reduction techniques in stochastic optimization \cite{johnson2013accelerating, allen2016variance, reddi2016stochastic,defazio2014saga}, 
there have been surging interests in developing 
\emph{variance-reduced} PG methods \cite{xu2017stochastic,papini2018stochastic,xu2019improved,xu2019sample,yuan2020stochastic}, which are shown to be faster.

In contrast to the empirical successes of PG methods, their theoretical convergence guarantees, especially \emph{non-asymptotic global} convergence  guarantees, have not been addressed satisfactorily until very recently \cite{fazel2018global,zhang2019global,wang2019neural,bhandari2019global,agarwal2019optimality}.  By \emph{non-asymptotic global} convergence, here we mean the convergence behavior of PG methods from any initialization, and the quality of the point they converge to (usually enjoys global optimality up to some compatible function approximation error due to policy parametrization), after a finite number of iterations/samples. These recent prominent  guarantees are normally beyond the folklore \emph{first-order} stationary-point convergence\footnotemark[1], as expected from a \emph{stochastic nonconvex optimization} perspective of solving RL with PG methods. Special landscapes of the RL objective, though nonconvex, have enabled the convergence to even globally optimal values. On the other hand, none of the aforementioned variance-reduced PG methods \cite{xu2017stochastic,papini2018stochastic,xu2019improved,xu2019sample,yuan2020stochastic} have been shown to enjoy these desired global convergence properties. It remains unclear whether these methods can converge to beyond first-order stationary policies. 

\footnotetext[1]{That is, finding a parameter $\theta$ such that $\|\nabla J(\theta)\|^2\leq \varepsilon$, where $J$ is the expected return. }

Motivated by these advances and the questions that remain to be answered,  we aim in this paper to improve the convergence of PG and natural PG (NPG) methods, and their variance-reduced variants,  under general smooth policy parametrizations. Our contributions are summarized as follows. 

\vspace{9pt}
\noindent{\textbf{Contributions.}} With a focus on the conventional Monte-Carlo-based PG methods, we propose a general framework for analyzing their \emph{global convergence}. 
Our contribution is three-fold: first, we establish the global convergence up to compatible function approximation  errors due to policy parametrization, for a variance-reduced PG method SRVR-PG \cite{xu2019sample}; second, we improve the global convergence of NPG methods established in \cite{agarwal2019optimality}, from $\mathcal{O}\left(\varepsilon^{-4}\right)$ to $\mathcal{O}\left(\varepsilon^{-3}\right)$;
third, we propose a new variance-reduced algorithm based on NPG, and establish its global convergence with an efficient sample-complexity. These improvements are based on a framework that integrates the advantages of previous analyses on (variance reduced) PG and NPG, and rely on a (mild) assumption that the Fisher information matrix induced by the policy parametrization is positive definite (see Assumption \ref{assump: strong convexity}). 
 A comparison of previous results and our improvements is laid out in Table \ref{table: summary of results}. 

\vspace{6pt}
\noindent{\textbf{Related Work.}}
\vspace{1pt} 

\noindent{\textbf{Global Convergence of (Natural) PG.}}
Recently, there has been a surging research interest in investigating the global convergence of PG and NPG methods, which is beyond the folklore convergence to first-order stationary policies. In the special case with linear dynamics and quadratic reward, \cite{fazel2018global} shows that PG methods with random search converge to the  globally optimal policy with linear rates. In \cite{zhang2019global}, with a simple reward-reshaping, PG methods have been shown to converge to the second-order stationary-point policies. \cite{bhandari2019global} shows that for finite-MDPs and several control tasks, the nonconvex RL objective has no suboptimal local minima. \cite{wang2019neural} prove that   (natural) PG methods converge to the globally optimal value when overparametrized neural networks are used for function approximation.  \cite{agarwal2019optimality} provides a fairly general characterization of global convergence for these methods, and a 
 basic sample complexity result for sample-based NPG updates.  It is also worth noting that trust-region policy optimization (TRPO) \cite{schulman2015trust},  as a variant of NPG, also enjoys global convergence with overparametrized neural networks \cite{liu2019neural}, and for regularized MDPs  \cite{shani2019adaptive}. 
     Very recently, for actor-critic algorithms, a series of non-asymptotic convergence results have also been established \cite{xu2020improving,xu2020non,wu2020finite,hong2020two},  with global convergence guarantees when natural PG/PPO are used in the actor step. 

\vspace{3pt} 


\vspace{-3pt}
\noindent{\textbf{Variance-Reduction (VR) for  PG.}} Conventional approaches to reduce the high variance in PG methods include using (natural) actor-critic algorithms \cite{konda2000actor,peters2008natural,bhatnagar2009natural}, and adding baselines  \cite{sutton2000policy,wu2018variance}. The idea of variance reduction (VR) is first proposed to accelerate stochastic minimization. VR algorithms such as SVRG \cite{johnson2013accelerating, allen2016variance, reddi2016stochastic}, SAGA \cite{defazio2014saga}, SARAH \cite{nguyen2017sarah}, and Spider \cite{fang2018spider} achieve acceleration over SGD in both convex and nonconvex settings. SVRG is also accelerated by applying a positive definite preconditioner that captures the curvature of the objective \cite{liu2019acceleration}.  Inspired by these successes in stochastic optimization, VR is also incorporated into PG methods \cite{xu2017stochastic}, with empirical validations for acceleration,  and analyzed rigorously in \cite{papini2018stochastic}.  Then, \cite{xu2019improved} improves the sample complexity of SVRPG, and \cite{xu2019sample} proposes a new SRVR-PG method that uses recursively updated semi-stochastic policy gradient, which leads to an improved sample complexity of $\cO(\varepsilon^{-1.5})$ over previous works. More recently, \cite{yuan2020stochastic} proposes a new STORM-PG method, which blends momentum in the update and matches the sample complexity of in \cite{xu2019sample}, and \cite{pham2020hybrid} applies the idea of SARAH and considers a more general setting with regularization. Finally, heavy-ball type of momentum has also been applied to PG methods \cite{huang2020momentum}. We highlight that all these sample complexity results are for first-order stationary-point convergence (which might have arbitrarily bad performance: see \eqref{equ: stationary convergence definition}), in contrast to the more desired global convergence guarantees (up to some function approximation errors that can be small) that we are interested in.   

%
%


\begin{table}[H]
\begin{center}
\begin{tabular}{cccc}	
   \begin{tabular}[c]{@{}c@{}}NPG\\\cite{agarwal2019optimality}\end{tabular} & \begin{tabular}[c]{@{}c@{}}NPG\\ \cite{wang2019neural}\end{tabular}   & \begin{tabular}[c]{@{}c@{}}TRPO\\ \cite{liu2019neural}\end{tabular}  & \begin{tabular}[c]{@{}c@{}}TRPO\\ \cite{shani2019adaptive}\end{tabular}                                                    \\  \hline \addlinespace[0.2cm]
    $\mathcal{O}(\varepsilon^{-4})$ & $\mathcal{O}(T_{TD}\varepsilon^{-2})$ \footnotemark[1]         & $\mathcal{O}(\varepsilon^{-8})$ & $\mathcal{O}(\varepsilon^{-4})$ \\ \hline
\end{tabular}
\vskip 0.5cm
\begin{tabular}{cccc}	
 \begin{tabular}[c]{@{}c@{}}NPG\\ \eqref{equ: NPG update_2}\end{tabular} & \begin{tabular}[c]{@{}c@{}}PG\\ \eqref{equ: PG update}\end{tabular} & \begin{tabular}[c]{@{}c@{}}SRVR-PG \cite{xu2019sample}\\ (Algorithm \ref{alg: SRVR-PG})\end{tabular}   & \begin{tabular}[c]{@{}c@{}}SRVR-NPG\\ (Algorithm \ref{alg: SRVR-NPG})\end{tabular}                                                     \\ \hline \addlinespace[0.2cm]
 $\cO(\varepsilon^{-3})$ &$\mathcal{O}({\sigma^2}{\varepsilon^{-4}})$  & $\mathcal{O}\left((W+\sigma^2)\varepsilon^{-3}\right)$         & $\mathcal{O}\left((W+\sigma^2)\varepsilon^{-2.5}+\varepsilon^{-3}\right)$ \\ \hline
\end{tabular}
\end{center}
\caption{Comparison of sample complexities of several methods to reach global optimality up to some compatible function approximation error (see \eqref{equ: compatible function approximation error}). Our results are listed in the second table (See App. \ref{app: previous results} for their derivations). We compare the number of trajectories to reach $\varepsilon-$optimality in expectation, up to some inherent error due to the function approximation for policy parametrization (see \eqref{equ: global convergence definition}). $\sigma^2$ is an upper bound for the variance of gradient estimator (see Assumption \ref{assump: variance}), and $W$ is an upper bound for the variance of importance weight (see Assumption \ref{assump: importance sampling}).} 
\label{table: summary of results}
\vspace{-5pt}
\end{table}
\footnotetext[1]{In \cite{wang2019neural}, $T_{TD}$ iterations of temporal difference updates are needed at each iteration, $T_{TD}$ can be large for wide neural networks. See App. \ref{app: previous results} for details.}
\vspace{-5pt}

\vspace{-10pt}
\section{Preliminaries}
\vspace{-5pt}


We first introduce some preliminaries regarding both the MDPs and policy gradient methods. 
 
 \vspace{-5pt}
\subsection{Markov Decision Processes}\label{sec:prelim_MDP}
\vspace{-5pt}

Consider a discounted Markov decision process  defined by a  tuple $(\cS,\cA,\PP,R,\gamma)$, where $\cS$ and $\cA$ denote the state and action spaces of the agent, $\PP(s'\given s,a):\cS\times\cA\to \cP(\cS)$ is the Markov kernel  that determines the transition probability from $(s,a)$ to state ${s}'$, $\gamma\in(0,1)$ is the discount factor, and $r:\cS\times\cA\to [-R,R]$ is the reward  function
of $s$ and $a$. 

At each time $t$, the agent executes an action $a_t\in\cA$ given the current state $s_t\in\cS$, following a possibly stochastic policy $\pi:\cS\to \cP(\cA)$, i.e., $a_t\sim \pi(\cdot\given s_t)$. 
Then, given the state-action pair $(s_t,a_t)$, the agent observes a reward $r_t=r(s_t,a_t)$. 
Thus, under any policy $\pi$, one can define the \emph{state-action value} function $Q^{\pi}:\cS\times\cA\to\RR$ as 
\$
Q^{\pi}(s,a):=\EE_{a_t\sim \pi(\cdot\given s_t),s_{t+1}\sim \PP(\cdot\given s_t,a_t)}\bigg(\sum_{t=0}^\infty \gamma^t r_t\bigggiven s_0=s,a_0=a\bigg).
\$
One can also define the \emph{state-value}  function $V^\pi:\cS\to \RR$, and  the \emph{advantage} function $A^\pi:\cS\times\cA\to \RR$, under policy $\pi$, as $V^\pi(s):=\EE_{a\sim \pi(\cdot\given s)}[Q^\pi(s,a)]$ and $A^\pi(s,a):=Q^\pi(s,a)-V^\pi(s)$, respectively. Suppose that the initial state $s_0$ is drawn from some distribution $\rho$. Then,  the goal of the agent is to find the optimal policy that maximizes the expected discounted return, namely,
\#\label{equ: def_obj}
\max_{\pi}~~J(\pi):=\EE_{s_0\sim\rho}[V^\pi(s_0)].
\#
In practice, both the state and action spaces $\cS$ and $\cA$ can be very large. Thus, the policy $\pi$ is usually parametrized as $\pi_\theta$ for some parameter $\theta\in\RR^d$, using, for example, deep neural networks. As such, the goal of the agent is to maximize $J(\pi_\theta)$ in the space of the parameter $\theta$, which naturally induces an optimization problem. Such a problem is in general nonconvex \cite{zhang2019global, agarwal2019optimality}, making it challenging to find the globally optimal policy. 

For notational convenience, let us denote $J(\pi_\theta)$ by $J(\theta)$. Many of the previous works focus on establishing stationary convergence of policy gradient methods. That is, finding a $\theta$ that satisfies 
\begin{align}
\label{equ: stationary convergence definition}
\|\nabla J(\theta)\|^2\leq \varepsilon.
\end{align}
Obviously, such a $\theta$ may not lead to a large $J(\theta)$. Instead, we are interested in finding a $\theta$ such that 
\begin{align}
\label{equ: global convergence definition}	
J^{\star} - J(\theta) \leq \cO(\sqrt{\varepsilon_{\text{bias}}})+\varepsilon,
\end{align}
where $J^{\star}=\max_{\pi} J(\pi)$, and the $\cO(\sqrt{\varepsilon_{\text{bias}}})$  term reflects the inherent error related  to  the possibly limited expressive power of the policy parametrization $\pi_{\theta}$ (see Assumption \ref{assump: compatible error} for the definition). 



\subsection{(Natural) Policy Gradient Methods}\label{sec:prelim_NPG}


To solve the optimization problem \eqref{equ: def_obj}, one standard way is via the policy gradient (PG) method \cite{sutton2000policy}. Specifically, let $\tau_i=\{s^i_0,a^i_0,s^i_1,\cdots\}$ denote the data of a sampled trajectory under policy $\pi_\theta$. Then, a stochastic PG ascent update is given as 
\begin{align}
\label{equ: PG update}
\theta^{k+1} = \theta^{k} + \eta \cdot \frac{1}{N}\sum_{i=1}^N g(\tau_i\given \theta^k),
\end{align}
where $\eta>0$ is a stepsize, $N$ is the number of trajectories, and $g(\tau_i \given\theta^k)$ estimates $\nabla J(\theta^k)$ using the trajectory $\tau_i$. Common unbiased estimators of PG include REINFORCE \cite{williams1992simple}, using the policy gradient theorem \cite{sutton1992reinforcement}, and GPOMDP \cite{baxter2001infinite}. The commonly used GPOMDP estimator will be given by 
	\#\label{equ:GPOMDP_surro}
	g(\tau_i\given \theta)= \sum_{h=0}^{\infty} \left(\sum_{t=0}^{h}\nabla_{\theta}\log \pi_{\theta}(a^i_t\given s^i_t)\right)\left(\gamma^h r(s^i_h, a^i_h)\right), 
	\#
	where $\nabla_{\theta}\log \pi_{\theta}(a^i_t\given s^i_t)$ is the \emph{score function}. If the expectation of this infinite sum exits, then \eqref{equ:GPOMDP_surro} becomes an unbiased estimate of the policy gradient of the objective $J(\theta)$ defined in \eqref{equ: def_obj}.
	This unbiasedness is established in App. \ref{sec:help_lemma} for completeness.

	In practice, a \emph{truncated} version of GPOMDP is used to approximate the infinite sum in \eqref{equ:GPOMDP_surro}, as  
\begin{align}
\label{equ: truncated GPOMDP estimator}
	g(\tau_i^H\given \theta) &= \sum_{h=0}^{H-1} \left(\sum_{t=0}^{h}\nabla_{\theta}\log \pi_{\theta}(a^i_t\given s^i_t)\right)\left(\gamma^h r(s^i_h, a^i_h)\right),
\end{align}
where $\tau_i^H=\{s^i_0,a^i_0,s^i_1,\cdots,s^i_{H-1},a^i_{H-1}, s^i_H\}$ is a truncation of the full trajectory $\tau_i$ of length $H$. \eqref{equ: truncated GPOMDP estimator} is thus a biased stochastic estimate of $\nabla J(\theta)$, with the bias being negligible for a large enough $H$. For notational simplicity, we denote the $H$-horizon trajectory distribution induced by the initial state distribution $\rho$ and  policy $\pi_\theta$ as $p^H_{\rho}(\cdot \given \theta)$, that is,
\[
p^H_{\rho}(\tau^H\given\theta) = \rho (s_0)\prod_{h=0}^{H-1} \pi_{\theta}(a_h\given s_h) \PP(s_{h+1}\given a_{h},s_{h}).
\]
Hereafter, unless otherwise stated, we refer to this \emph{$H$-horizon trajectory}  simply as \emph{trajectory}, drawn from $p^H_{\rho}(\cdot\given\theta)$. 

As a significant variant of PG, NPG \cite{kakade2002natural} also incorporates a preconditioning matrix $F_{\rho}(\theta)$, leading to the following update
\#
&F_{\rho}(\theta)=\EE_{s\sim d^{\pi_\theta}_{\rho}}[F_s(\theta)], \qquad \theta^{k+1}
=\theta^k+\eta\cdot F^{\dagger}_\rho(\theta^k)\nabla J(\theta^k), \label{equ: NPG update}
\#
where $F_s(\theta) = \EE_{a\sim \pi_{\theta}(\cdot\given s)}\left[\nabla_{\theta}\log\pi_{\theta}(a\given s)\nabla_{\theta}\log\pi_{\theta}(a\given s)^\top \right]$ is the Fisher information matrix of $\pi_{\theta}(\cdot\given s)\in \cP(\cA)$, $F^{\dagger}_\rho(\theta^k)$ is the Moore-Penrose pseudoinverse of $F_\rho(\theta^k)$, and $d^{\pi_\theta}_{\rho}\in\cP(\cS)$ is the state visitation measure induced by policy $\pi_\theta$ and initial distribution $\rho$, which is  defined as
\[
d^{\pi_\theta}_{\rho}(s)\coloneqq (1-\gamma)\EE_{s_0\sim\rho}\sum_{t=0}^\infty\gamma^t \PP(s_t=s\given s_0,\pi_\theta).
\]
The NPG update \eqref{equ: NPG update} can also be written as \cite{kakade2002natural,agarwal2019optimality}
\#\label{equ: NPG update_2}
\theta^{k+1}=\theta^k+\eta \cdot w^k,\quad~~ \text{with~~~~} w^k\in\argmin_{w\in\RR^d}~~~L_{\nu^{\pi_{\theta}}_{\rho}}(w; \theta),
\#
where $L_{\nu^{\pi_{\theta}}_{\rho}}(w; \theta)$ is the \textit{compatible function approximation error} defined by
\begin{align}
\label{equ: compatible function approximation error}
	L_{\nu^{\pi_{\theta}}_{\rho}}(w; \theta)=\EE_{(s,a)\sim \nu^{\pi_{\theta}}_{\rho}}\left[\big(A^{\pi_{\theta}}(s,a)-(1-\gamma)w^\top\nabla_{\theta}\log\pi_{\theta}(a\given s)\big)^2\right]. 
\end{align}
Here, $\nu^{\pi_{\theta}}_{\rho}(s,a) = d^{\pi_\theta}_{\rho}(s)\pi(a\given s)$ is the \emph{state-action} visitation measure  induced by $\pi_\theta$ and initial state distribution $\rho$, which can also be written as 
\begin{align}
\label{equ: nu with nu_0}
\nu^{\pi_{\theta}}_{\rho}(s,a)\coloneqq (1-\gamma)\EE_{s_0\sim\rho}\sum_{t=0}^\infty\gamma^t \PP(s_t=s,a_t=a\given s_0,\pi_\theta).
\end{align}
For convenience, we will denote $\nu^{\pi_{\theta}}_{\rho}$ by $\nu^{\pi_{\theta}}$ hereafter. 
In other words, the NPG update direction $w^k$ is given by the minimizer of a  stochastic optimization problem. In practice, one obtains an approximate NPG update direction $w^k$ by SGD (see Procedure \ref{alg: SGD for NPG subproblem}).



Regarding the NPG update \eqref{equ: NPG update_2}, we make the following standing assumption on the Fisher information matrix induced by $\pi_\theta$ and $\rho$. 
\begin{assumption}
\label{assump: strong convexity}
For all $\theta\in\Rd$, 
 the Fisher information matrix induced by policy $\pi_{\theta}$ and initial  state distribution $\rho$ satisfies
\$
 F_{\rho}(\theta)=\EE_{(s,a)\sim \nu^{\pi_{\theta}}_{\rho}}\left[\nabla_{\theta}\log\pi_{\theta}(a\given s)\nabla_{\theta}\log\pi_{\theta}(a\given s)^\top \right]\succcurlyeq \mu_F \cdot I_d
\$
for some constant $\mu_F>0$. 
\end{assumption}
Assumption \ref{assump: strong convexity} essentially states that $F_{\rho}(\theta)$ behaves well as a preconditioner in the NPG update \eqref{equ: NPG update_2}. 
This is a common (and minimal) requirement for the convergence of preconditioned algorithms in both convex and nonconvex settings in the optimization realm, for example, the quasi-Newton algorithms \cite{broyden1970convergence, fletcher1970new, goldfarb1970family, shanno1970conditioning}, and their stochastic variants \cite{byrd2016stochastic,moritz2016linearly, gower2016stochastic, wang2017stochastic, liu2019acceleration}. In the RL realm, one common example of policy parametrizations that can satisfy  this assumption is the Gaussian policy \cite{williams1992simple,duan2016benchmarking,papini2018stochastic,xu2019sample}, where $\pi_\theta(\cdot\given s)=\cN(\mu_\theta(s),\Sigma)$ with mean parametrized linearly as $\mu_\theta(s)=\phi(s)^\top \theta$, where $\phi(s)$ denotes some feature matrix of proper dimensions, $\theta$ is the coefficient vector, and $\Sigma\succ 0$ is some fixed covariance matrix. In this case, the Fisher information matrix at each $s$ becomes $\phi(s)\Sigma^{-1}\phi(s)^{\top}$, independent of $\theta$, and is uniformly lower bounded (positive definite sense) if $\phi(s)$ is full-row-rank, namely, the features expanded by $\theta$ are linearly independent, which is a common requirement for linear function approximation settings  \cite{tsitsiklis1997analysis,melo2008analysis,sutton2009fast}.  
See App. \ref{sec:append_justify_Fisher} for more detailed justifications, as well as discussions on more general policy parametrizations. 

In the pioneering NPG work \cite{kakade2002natural}, $F(\theta)$ is directly assumed to be positive definite. So is in the follow-up works on  natural actor-critic algorithms \cite{peters2008natural,bhatnagar2009natural}. 
In fact, this way, $F(\theta)$ will define a valid Riemannian metric on the  parameter space, which has been used for interpreting the desired convergence properties of natural gradient methods  \cite{amari1998natural,martens2014new}. In a recent version of \cite{agarwal2019optimality}, a relevant assumption (specifically, Assumption 6.5, item 3) is made to establish the global convergence of NPG, in which it is assumed that $\lambda_{\textrm{min}}(F_{\rho}(\theta))$ is not too small compared with the Fisher information matrix induced by a fixed comparator policy. this can be implied by our Assumption \ref{assump: strong convexity}. To sum up, the positive definiteness on the Fisher preconditioning matrix is common and not very restrictive.  

In Sec. \ref{sec: theory}, we shall see that under Assumption \ref{assump: strong convexity}, the stationary convergence of NPG can be analyzed, and NPG enjoys a better sample complexity of $\mathcal{O}(\varepsilon^{-3})$ in terms of its global convergence, compared with the existing sample complexity of $\mathcal{O}(\varepsilon^{-4})$ in \cite{agarwal2019optimality}. In addition, interestingly, PG and its variance-reduced version SRVR-PG also enjoy global convergence,  although the Fisher information matrix does not appear explicitly in their updates.

\vspace{-5pt}
\section{Variance-Reduced Policy Gradient Methods}
\label{sec:variance reduction}
\vspace{-5pt}

Recently, \cite{xu2019sample} proposes an algorithm called Stochastic Recursive Variance Reduced Policy Gradient (SRVR-PG, see Algorithm \ref{alg: SRVR-PG}), which applies variance-reduction on PG. It achieves a sample complexity of $\cO(\varepsilon^{-1.5})$ to find an $\varepsilon-$stationary point, compared with the $\cO(\varepsilon^{-2})$ sample complexity of stochastic PG. However, it remains unclear whether SRVR-PG converges globally. In this work, we provide an affirmative answer to this question by showing that SRVR-PG has a sample complexity of $\cO(\varepsilon^{-3})$ to find an $\varepsilon-$optimal policy, up to some compatible function approximation error due to policy parametrization.  


 We also propose a new algorithm called SRVR-NPG to incorporate variance reduction into NPG, which is described in Algorithm \ref{alg: SRVR-NPG}. In Sec. \ref{sec: theory}, we provide a sample complexity for its global convergence, which is comparable to our improved NPG result.

%

\begin{algorithm}[t]    
\caption{Stochastic Recursive Variance Reduced Natural Policy Gradient (SRVR-NPG)}
\label{alg: SRVR-NPG}
    \textbf{Input:} number of epochs $S$, epoch size $m$, stepsize $\eta$, batch size $N$, minibatch size $B$, truncation horizon $H$, initial parameter $\theta^0_m=\theta_0\in\Rd.$
    \begin{algorithmic}[1]
        \For{$j\leftarrow 0,...,S-1$}{}
        \State{$\theta^{j+1}_0=\theta^j_m$;}
        \State{Sample $\{\tau^H_i\}_{i=1}^N$ from $p_{\rho}^H(\cdot\given \theta^{j+1}_0)$ and calculate $u^{j+1}_0=\frac{1}{N}\sum_{i=1}^N g(\tau^H_i\given \theta^{j+1}_0)$;} 
        \State{$w^{j+1}_0 = \texttt{SRVR-NPG-SGD}(\nu^{\pi_{\theta^{j+1}_0}}, \pi_{\theta^{j+1}_0}, u^{j+1}_0)$;}
        \Comment{$w^{j+1}_0\approx w^{j+1}_{0,\star} = F^{-1}_{\rho}(\theta^{j+1}_0)u^{j+1}_0$; }
        \State{$\theta^{j+1}_{1} = \theta^{j+1}_0+\eta w^{j+1}_0;$}
        \For{$t \leftarrow 1,...,m-1$}{}
        \State{Sample $B$ trajectories $\{\tau^H_j\}_{j=1}^{B}$ from $p_{\rho}^H(\cdot|\theta^{j+1}_t)$;} 
        \State{$u^{j+1}_t=u^{j+1}_{t-1}+\frac{1}{B}\sum_{j=1}^B \left(g(\tau^H_j\given \theta^{j+1}_t)-g_w(\tau^H_j\given \theta^{j+1}_{t-1})\right)$;}
         \State{$w^{j+1}_t = \texttt{SRVR-NPG-SGD}(\nu^{\pi_{\theta^{j+1}_t}}, \pi_{\theta^{j+1}_t}, u^{j+1}_t)$;}
         \Comment{$w^{j+1}_t\approx w^{j+1}_{t,\star} = F^{-1}_{\rho}(\theta^{j+1}_t)u^{j+1}_t$; }
        \State{$\theta^{j+1}_{t+1} = \theta^{j+1}_t+\eta w^{j+1}_t;$}
        \EndFor
        \EndFor
        \State\Return{$\theta_{\text{out}}$ chosen uniformly from $\{\theta\}_{j=1,...,S; t=0,...,m-1.}$}
    \end{algorithmic}
\end{algorithm}

In line 8 of Algorithm \ref{alg: SRVR-NPG}, $g_w(\tau^H_j|\theta^{j+1}_{t-1})$ is a weighted gradient estimator given by 
\begin{align}
\label{equ: weighted GPOMDP}
g_w(\tau^H_j\given \theta^{j+1}_{t-1}) &= \sum_{h=0}^{H - 1}w_{0:h} (\tau^H_j \given  \theta^{j+1}_{t-1}, \theta^{j+1}_{t}) \left(\sum_{t=0}^{h}\nabla_{\theta}\log \pi_{\theta}(a^i_t\given s^i_t)\right)\left(\gamma^h r(s^i_h, a^i_h)\right),	
\end{align}
where the importance weight factor $w_{0:h}(\tau^H_j|\theta^{j+1}_{t-1}, \theta^{j+1}_{t})$  is defined by 
\begin{align} 
\label{equ: importance weight}
w_{0:h}(\tau^H_j\given \theta^{j+1}_{t-1}, \theta^{j+1}_{t})= \prod_{h'=0}^h \frac{\pi_{\theta^{j+1}_{t-1}}(a_{h'}\given s_{h'})}{\pi_{\theta^{j+1}_t}(a_{h'}\given s_{h'})}.
\end{align} 
%
This importance sampling makes $u^{j+1}_t$ an unbiased estimator of $\nabla J^H(\theta^{j+1}_t)$. 

In lines 4 and 8 of Algorithm \ref{alg: SRVR-NPG}, $w^{j+1}_t$ is produced by  \texttt{SRVR-NPG-SGD} (see Procedure \ref{alg: SGD for SRVR-NPG subproblem}), which applies SGD\footnotemark[1] to solve the following subproblem:
\begin{align}
\label{equ: SRVR-NPG subproblem}
w^{j+1}_{t}\approx \argmin_{w}\left\{\mathbb{E}_{(s,a)\sim \nu^{j+1}_t }\left[ \left(w^T\nabla_{\theta}\log \pi_{\theta^{j+1}_t}(a\given s)\right)^2\right]-2\langle  w, u^{j+1}_t\rangle\right\},
\end{align}
where $\nu^{j+1}_{t}$ is the state-action visitation measure induced by $\pi_{\theta^{j+1}_t}$. The exact update direction given by \eqref{equ: SRVR-NPG subproblem} is $F^{-1}_{\rho}(\theta^{j+1}_t)u^{j+1}_t$, and as in NPG, $F_{\rho}(\theta^{j+1}_t)$ also serves as a preconditioner. 



\footnotetext[1]{Following \cite{agarwal2019optimality}, we apply SGD \cite{bach2013non} to make a fair comparison. One can also apply the SA algorithm \cite{nemirovski2009robust} and AC-SA algorithm \cite{ghadimi2012optimal}.}

\vspace{-5pt}
\section{Theoretical Results}
\label{sec: theory}
\vspace{-5pt}



Before presenting the global convergence results, we first introduce some standard assumptions.


\begin{assumption}
\label{assump: variance}
The truncated GPOMDP estimator $g(\tau^H\given \theta)$ defined in \eqref{equ: truncated GPOMDP estimator} satisfies $\text{Var}\left(g(\tau^H\given\theta)\right)\coloneqq \E [\|g(\tau^H\given \theta) - \E[g(\tau^H\given \theta)]\|^2]\leq \sigma^2$ for any $\theta$ and $\tau^H\sim p^H_{\rho}(\cdot\given \theta)$.  
\end{assumption}
\begin{assumption}
\label{assump: conditions on score function}
\begin{enumerate}
    \item $\|\nabla_{\theta}\log \pi_{\theta}(a\given s)\|\leq G$ for any $\theta$ and $(s,a)\in\cS\times \cA$.
    \item $\|\nabla_{\theta}\log \pi_{\theta_1}(a\given s)-\nabla_{\theta}\log \pi_{\theta_2}(a\given s)\|\leq M\|\theta_1-\theta_2\|$ for any $\theta_1, \theta_2$ and $(s,a)\in\cS\times \cA$.
\end{enumerate}
\end{assumption}
\begin{assumption}
\label{assump: importance sampling}
For the importance weight $w_{0:h}(\tau^H|\theta_1, \theta_2)$ \eqref{equ: importance weight}, there exists $W>0$ such that 
\[
\text{Var}(w_{0:h}\left(\tau^H\given\theta_1, \theta_2)\right)\leq W, \,\,\,\forall \theta_1,\theta_2\in\Rd, \tau^H\sim p^H_{\rho}(\cdot\given\theta_2).
\]
\end{assumption}
Assumptions \ref{assump: variance}, \ref{assump: conditions on score function} and \ref{assump: importance sampling} are standard in the analysis of PG methods and their variance reduced variants \cite{agarwal2019optimality, papini2018stochastic, xu2019improved,xu2019sample}. They can be verified for simple policy parametrizations such as Gaussian policies; see \cite{papini2018stochastic, pirotta2013adaptive,cortes2010learning} for more justifications.  


Following the Assumption 6.5 of \cite{agarwal2019optimality}, we assume that the policy parametrization $\pi_{\theta}$ achieves a good function approximation, as measured by the \textit{transferred compatible function approximation error}.

\begin{assumption}
\label{assump: compatible error}
For any $\theta\in\Rd$, the \textit{transferred compatible function approximation error} satisfies 
\begin{align}
\label{equ: minimal compatible function approximation error}
L_{\nu^{\star}}(w^{\theta}_{\star}; \theta)= \EE_{(s,a)\sim \nu^{\star}}\left[\big(A^{\pi_{\theta}}(s,a)-(1-\gamma)(w^{\theta}_{\star})^\top\nabla_{\theta}\log\pi_{\theta}(a\given s)\big)^2\right]\leq \varepsilon_{\text{bias}},
\end{align} 
where $\nu^{\star}(s,a) = d^{\pi^{\star}}_{\rho}(s) \cdot \pi^{\star}(a \given s)$ is the state-action distribution induced by an optimal policy $\pi^{\star}$ that maximizes $J(\pi)$, and $w^{\theta}_{\star} = \argmin_{w\in\Rd} L_{\nu^{\pi_{\theta}}_{\rho}}(w; \theta)$ is the exact NPG update direction at $\theta$.
\end{assumption}
$\varepsilon_{\text{bias}}$ reflects the error when approximating the advantage function from the score function, it measures the capacity of the parametrization $\pi_{\theta}$. When $\pi_{\theta}$ is the softmax parametrization, we have $\varepsilon_{\text{bias}}=0$ \cite{agarwal2019optimality}. When $\pi_{\theta}$ is a restricted parametrization, $\varepsilon_{\text{bias}}$ is often positive as $\pi_{\theta}$ may not contain all stochastic policies.  For rich neural parametrizations, $\varepsilon_{\text{bias}}$ is very small \cite{wang2019neural}. 



%

%

\vspace{-5pt}
\subsection{A General Framework for Global Convergence}
\vspace{-5pt}

 Inspired by the global convergence analysis of NPG in \cite{agarwal2019optimality}, we present a general framework that relates the global convergence rates of these algorithms to i) their stationary convergence rate on $ J(\theta)$, and ii) the difference between their update directions and exact NPG update directions. 
%
%

\begin{proposition}
\label{prop: global convergence}
Let $\{\theta^k\}_{k=1}^K$ be generated by a general update of the form
\[
\theta^{k+1} = \theta^k +\eta w^k, \,\,\,\,\, k = 0,1,...K-1.
\]
Furthermore, let $w^k_{\star} = F_{\rho}^{-1}(\theta^k)\nabla  J(\theta^k)$ be the exact NPG update direction at $\theta^k$. Then, we have
\begin{align}
J(\pi^{\star})-\frac{1}{K}\sum_{k=0}^{K-1}J(\theta^k)
&\leq \frac{\sqrt{\varepsilon_{\text{bias}}}}{1-\gamma} + \frac{1}{\eta K} \mathbb{E}_{s\sim d^{\pi^{\star}}_\rho} \left[\text{KL}\left(\pi^{\star}(\cdot\given s)|| \pi_{\theta^0}(\cdot\given s)\right)\right]\nonumber\\
&\,\,\, +\frac{M\eta}{2K}\sum_{k=0}^{K-1}\|w^k\|^2 + \frac{G}{K}\sum_{k=0}^{K-1} \|w^k-w^k_{\star}\|, \label{equ: global convergence}
\end{align}
where $\pi^{\star}$ is an optimal policy that maximizes $J(\pi)$. 

\end{proposition}


The detailed proof of this global convergence framework can be found in \ref{app: global}. To obtain a high level idea, one first starts from the $M-$smoothness of the score function to get
\begin{align*}
&\mathbb{E}_{s\sim d^{\pi^{\star}}_{\rho}} \left[\text{KL}\left(\pi^{\star}(\cdot\given s)|| \pi_{\theta^{k}}(\cdot\given s)\right)-\text{KL}\left(\pi^{\star}(\cdot\given s)|| \pi_{\theta^{k+1}}(\cdot \given s)\right)\right]\\
&\geq \eta \mathbb{E}_{s\sim d^{\pi^{\star}}_{\rho}}\mathbb{E}_{a\sim \pi^{\star} (\cdot \given s)} [\nabla_{\theta}\log \pi_{\theta^k}(a\given s)\cdot w^k_{\star}] \\
&\,\,\, + \eta \mathbb{E}_{s\sim d^{\pi^{\star}}_{\rho}}\mathbb{E}_{a\sim \pi^{\star} (\cdot \given s)} [\nabla_{\theta}\log \pi_{\theta^k}(a\given s)\cdot (w^k - w^k_{\star})] -\frac{M\eta^2}{2}\|w^k\|^2.
\end{align*}
On the other hand, the renowned Performance Difference Lemma \cite{kakade2002approximately} tells us that 
\begin{align*}
\mathbb{E}_{s\sim d^{\pi^{\star}}_{\rho}}\mathbb{E}_{a\sim \pi^{\star} (\cdot \given s)} [A^{\pi_{\theta^k}}(s,a)] = (1-\gamma) \left(J^{\star}-J(\theta^k)\right).
\end{align*}
To connect the advantage term $\mathbb{E}_{s\sim d^{\pi^{\star}}_{\rho}}\mathbb{E}_{a\sim \pi^{\star} (\cdot \given s)} [A^{\pi_{\theta^k}}(s,a)]$ with the inner product term $\mathbb{E}_{s\sim d^{\pi^{\star}}_{\rho}}\mathbb{E}_{a\sim \pi^{\star} (\cdot \given s)} [\nabla_{\theta}\log \pi_{\theta^k}(a\given s)\cdot w^k_{\star}]$, we invoke Assumption \ref{assump: compatible error}:
\[
 \mathbb{E}_{s\sim d^{\pi^{\star}}_{\rho}}\mathbb{E}_{a\sim \pi^{\star} (\cdot \given s)}\left[\big(A^{\pi_{\theta}}(s,a)-(1-\gamma)(w^{\theta}_{\star})^\top\nabla_{\theta}\log\pi_{\theta}(a\given s)\big)^2\right]\leq \varepsilon_{\text{bias}}, \quad \text{for any $\theta\in\Rd$.}
\]
The final result follows from a telescoping sum on $k = 0,1,...,K-1$.

Several remarks are in order. The first term on the right-hand side of \eqref{equ: global convergence} reflects the function approximation error due to the parametrization $\pi_{\theta}$, and the second term is  of the form $\cO(\frac{1}{K})$. The third term depends on the stationary convergence. With Assumption \ref{assump: strong convexity}, it can be shown that\footnotemark[1] $\frac{1}{K}\sum_{k=0}^{K-1}\EE[\|w^k\|^2]\rightarrow 0$ for both NPG and SRVR-NPG. The proof follows from an optimization perspective and is inspired by the stationary convergence analysis of stochastic PG (see App. \ref{sec: stationary convergence}).
\footnotetext[1]{The stationary convergence of SRVR-PG has been established in \cite{xu2019sample}. }

With Assumption \ref{assump: strong convexity}, we can also show that the last term of \eqref{equ: global convergence} is small. Take stochastic PG as an example; then, we have $w^k = \frac{1}{N}\sum_{i=1}^N g(\tau^H_i|\theta^k)$, and
\begin{align*}
\frac{1}{K}\sum_{k=0}^{K-1}\|w^k - w^k_{\star}\|
&\leq \frac{1}{K}\sum_{k=0}^{K-1}\|w^k-\nabla {J}(\theta^k)\| + \frac{1}{K}\sum_{k=0}^{K-1}\left(1+\frac{1}{\mu_F}\right)\|\nabla {J}(\theta^k)\|.
\end{align*}
When $H$ and $N$ are large enough, $w^k$ is a low-variance estimator of $\nabla  J^H(\theta^k)$, and $\nabla  J^H(\theta^k)$ is close to $\nabla  J(\theta^k)$, this makes the first term above small. The second term also goes to $0$ as $\theta^k$ approaches stationarity. 

\vspace{-5pt}
\subsection{Global Convergence Results}
\vspace{-5pt}

By applying Proposition \ref{prop: global convergence} on the PG, NPG, SRVR-PG, and SRVR-NPG updates and analyzing their stationary convergence, we obtain their global convergence rates. In the following, we only keep the dependences on $\sigma^2$ (the variance of the gradient estimator), $W$ (variance of importance weight), $\frac{1}{1-\gamma}$ (the effective horizon) and $\varepsilon$ (target accuracy). The specific choice of the parameters and sample complexities, as well as the proof, can be found in the appendix. 

%

\begin{theorem}
\label{thm: PG global convergence}
In the stochastic PG \eqref{equ: PG update} with the truncated GPOMDP estimator \eqref{equ: truncated GPOMDP estimator}, take $\eta=\frac{1}{4L_J}$, $K=\cO\left(\frac{1}{(1-\gamma)^{2}\varepsilon^2}\right)$, $N=\cO\left(\frac{\sigma^2}{\varepsilon^2}\right)$, and $H =\cO\left(\log(\frac{1}{(1-\gamma)\varepsilon})\right)$. Then, we have
\begin{align*}
\begin{split}
J(\pi^{\star})-\frac{1}{K}\sum_{k=0}^{K-1} \EE[J(\theta^k)]&\leq \frac{\sqrt{\varepsilon_{\text{bias}}}}{1-\gamma}+\varepsilon.
\end{split}
\end{align*}
In total, stochastic PG samples $\mathcal{O}\left(\frac{\sigma^2}{(1-\gamma)^2\varepsilon^4}\right)$ trajectories. 
\end{theorem}
\begin{remark}
$L_J = \frac{MR}{(1-\gamma)^2}$ is the Lipschitz constant of $\nabla J$, see Lemma \ref{lem: smoothness of objective} for details.
\end{remark}
\begin{remark}
Theorem \ref{thm: PG global convergence} improves the result of \cite[Thm. 6.11]{agarwal2019optimality} from (impractical) full gradients to sample-based stochastic gradients. 


\end{remark}

\begin{theorem}
\label{thm: NPG global convergence}
In the NPG update \eqref{equ: NPG update_2}, let us apply $\cO\left(\frac{1}{(1-\gamma)^4\varepsilon^2}\right)$ iterations of SGD as in Procedure \ref{alg: SGD for NPG subproblem} to obtain an update direction. In addition, take $\eta = \frac{\mu_F^2}{4G^2L_J}$ and $K=\cO\left(\frac{1}{(1-\gamma)^2\varepsilon}\right)$. Then,
\begin{align*}
\begin{split}
J^{\star}-\frac{1}{K}\sum_{k=0}^{K-1} \EE[J(\theta^k)]&\leq \frac{\sqrt{\varepsilon_{\text{bias}}}}{1-\gamma}+\varepsilon.
\end{split}
\end{align*}
In total, NPG samples $\mathcal{O}\left(\frac{1}{(1-\gamma)^6\varepsilon^3}\right)$ trajectories. 
\end{theorem}
\begin{remark}
Compared with \cite[Coro. 6.10]{agarwal2019optimality}, Theorem \ref{thm: NPG global convergence} improves the sample complexity of NPG by $\cO(\varepsilon^{-1})$. This is because our stationary convergence analysis on NPG allows for a constant stepsize $\eta$, while \cite[Coro. 6.10]{agarwal2019optimality} applies a stepsize of $\eta=\mathcal{O}(1/\sqrt{K})$.  It is worth noting that the $\cO(\sqrt{\varepsilon_{\text{bias}}})$ term is the same as in \cite{agarwal2019optimality}, and we also apply the average SGD \cite{bach2013non} to solve the NPG subproblem \eqref{equ: NPG update_2}. 
\end{remark}

\begin{theorem}
\label{thm: SRVR-PG global convergence}
In SRVR-PG (Algorithm \ref{alg: SRVR-PG}), take $\eta=\frac{1}{8L_J}$, $S=\mathcal{O}\left(\frac{1}{(1-\gamma)^{2.5}\varepsilon}\right)$, $m=\mathcal{O}\left(\frac{(1-\gamma)^{0.5}}{\varepsilon}\right)$, $B=\mathcal{O}\left(\frac{W}{(1-\gamma)^{0.5}\varepsilon}\right)$, $N=\mathcal{O}\left(\frac{\sigma^2}{\varepsilon}\right)$, and $H =\cO\left(\log(\frac{1}{(1-\gamma)\varepsilon})\right)$. Then, we have
\begin{align*}
\begin{split}
J^{\star}-\frac{1}{Sm}\sum_{s=0}^{S-1}\sum_{t=0}^{m-1} \EE[J(\theta^{j+1}_t)]&\leq \frac{\sqrt{\varepsilon_{\text{bias}}}}{1-\gamma} +\varepsilon.
\end{split}
\end{align*}
In total, SRVR-PG samples $\mathcal{O}\left(\frac{W+\sigma^2}{(1-\gamma)^{2.5}\varepsilon^3}\right)$ trajectories. 
\end{theorem}

\begin{remark}
Theorem \ref{thm: SRVR-PG global convergence} establishes the global convergence of SRVR-PG proposed in \cite{xu2019sample}, where only  stationary convergence is shown. Also, compared with stochastic PG, SRVR-PG enjoys a better sample complexity thanks to its faster stationary convergence.
\end{remark}

\begin{theorem}
\label{thm: SRVR-NPG global convergence}
In SRVR-NPG (Algorithm \ref{alg: SRVR-NPG}), let us apply $\cO\left(\frac{1}{(1-\gamma)^4\varepsilon^2}\right)$ iterations of SGD as in Procedure \ref{alg: SGD for SRVR-NPG subproblem} to obtain an update direction. In addition, take $\eta=\frac{\mu_F}{16L_J}$, $S=\cO\left(\frac{1}{(1-\gamma)^{2.5}\varepsilon^{0.5}}\right)$, $m= \cO\left(\frac{(1-\gamma)^{0.5}}{\varepsilon^{0.5}}\right)$, $B=\cO\left(\frac{W}{(1-\gamma)^{0.5}\varepsilon^{1.5}}\right)$, $N=\cO\left(\frac{\sigma^2}{\varepsilon^2}\right)$, and $H =\cO\left(\log(\frac{1}{(1-\gamma)\varepsilon})\right)$. Then,
\begin{align*}
\begin{split}
J^{\star}-\frac{1}{Sm}\sum_{s=0}^{S-1}\sum_{t=0}^{m-1} \EE[J(\theta^{j+1}_t)]&\leq \frac{\sqrt{\varepsilon_{\text{bias}}}}{1-\gamma}+\varepsilon.
\end{split}
\end{align*}
In total, SRVR-NPG samples $\cO\left(\frac{W+\sigma^2}{(1-\gamma)^{2.5}\varepsilon^{2.5}} + \frac{1}{(1-\gamma)^{6}\varepsilon^{3}}\right)$ trajectories. 
\end{theorem}

\begin{remark}
Compared with SRVR-PG, our SRVR-NPG has a better dependence on $W$ and $\sigma^2$, which could be large in practice (especially $W$). The current sample complexity of SRVR-NPG is not better than our (improved) result of NPG since, in our analysis, the advantage of variance reduction is offset by the cost of solving the subproblems.
\end{remark}

\section{Numerical Experiments}
\label{sec: experiments}
In this section, we compare the numerical performances of stochastic PG, NPG, SRVR-PG, and SRVR-NPG. Specifically, we test on benchmark reinforcement learning environments Cartpole and Mountain Car. Our implementation is based on the implementation of SRVPG\footnotemark[1] and SRVR-PG\footnotemark[2], and can be found in the supplementary material.

\footnotetext[1]{\url{https://github.com/Dam930/rllab}}
\footnotetext[2]{\url{https://github.com/xgfelicia/SRVRPG}}

For both tasks, we apply a Gaussian policy of the form $\pi_{\theta}(a\given s) = \frac{1}{\sqrt{2\pi}}\exp\left(-\frac{(\mu_{\theta}(s) - a)^2}{2\sigma^2}\right)$
where the mean $\mu_{\theta}(s)$ is modeled by a neural network with Tanh as the activation function. 

For the Cartpole problem, we apply a neural network of size $32\times 1$ and a horizon of $H = 100$. In addition, each training algorithm uses $5000$ trajectories in total. For the Mountain Car problem, we apply a neural network of size $64\times 1$ and take $H = 1000$. $3000$ trajectories are allowed for each algorithm. The numerical performance comparison, as well as the settings of algorithm-specific parameters, can be found in Figures \ref{fig: cartpole} and \ref{fig: mountain car}. In App. \ref{app: implementation details}, we provide more implementation details.


\newlength{\halfwidth}
\setlength{\halfwidth}{\dimexpr 0.45\textwidth-\tabcolsep}

\begin{figure}[htbp]
\centering
\begin{tabular}{@{}p{\halfwidth}p{\halfwidth}@{}}
\centering
  \raisebox{-\height}{\includegraphics[width=1\linewidth]{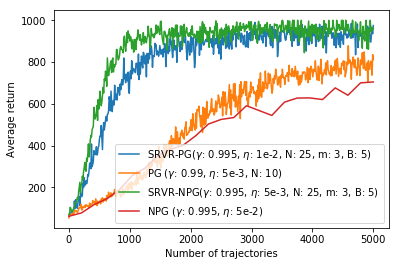}}&
    \raisebox{-\height}{\includegraphics[width=1\linewidth]{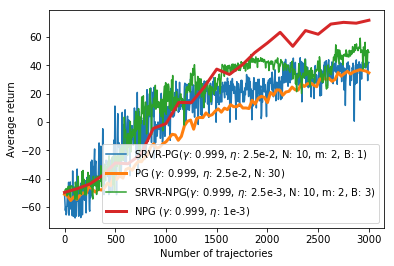}}\\
  \caption{Numerical Performances on Cartpole. For PG, SRVR-PG and SRVR-NPG, we report the undiscounted average return averaged over 10 runs. For NPG, we report the averaged return over 40 runs. Overall, SRVR-NPG has the best performance.}
\label{fig: cartpole}&
  \caption{Numerical Performances on Mountain Car. For PG, SRVR-PG and SRVR-NPG, we report the undiscounted average return averaged over 10 runs. For NPG, we report the averaged return over 40 runs. Overall, NPG has the best performance.}
\label{fig: mountain car}
\end{tabular}
\end{figure}
\vspace{-5pt}
\section{Concluding Remarks}
\label{sec:conclusions}
\vspace{-5pt}

In this work, we have introduced a framework for analyzing the global convergence of (natural) PG methods and their variance-reduced variants, under the assumption that the Fisher information matrix is positive definite. We have established the sample complexity for the global convergence of stochastic PG and its variance-reduced variant SRVR-PG, and improved the sample complexity of NPG. In addition, we have introduced SRVR-NPG, which incorporates variance-reduction into NPG, and enjoys both  global convergence guarantee and  an efficient sample complexity. Our improved analysis hinges on exploiting the advantages of previous analyses on (variance reduced) PG and NPG methods, which may be of independent interest, and can be used to design faster variance-reduced NPG methods in the future.  
  

\newpage
\onecolumn 

\section*{Broader Impact}
 

The results of this paper improves the performance of policy-gradient methods for reinforcement learning, as well as our understanding to the existing methods. Through reinforcement learning, our study will also benefit several research communities such as machine learning and robotics. We do not believe that the results in this work will cause any ethical issue, or put anyone at a disadvantage in our society.

\section*{Acknowledgements}
Yanli Liu and Wotao Yin were partially supported by the Office of Naval Research (ONR) Grant N000141712162. Yanli Liu was also supported by UCLA Dissertation Year Fellowship.  Kaiqing Zhang and Tamer Ba\c{s}ar were supported in part by the US Army Research Laboratory (ARL) Cooperative Agreement W911NF-17-2-0196, and in part by the Office of Naval Research (ONR) MURI Grant N00014-16-1-2710.

We would like to thank Rui Yuan for his suggestions to improve the proof of Lemma B.1 and Proposition G.1.
 


\bibliographystyle{plain}
\bibliography{main}


\newpage
\appendix
\onecolumn

~\\
\centerline{{\fontsize{13.5}{13.5}\selectfont \textbf{Supplementary Materials for ``An Improved Analysis of (Variance-}}}

\vspace{6pt}
 \centerline{\fontsize{13.5}{13.5}\selectfont \textbf{
Reduced) Policy Gradient and Natural Policy Gradient Methods''}}
 \vspace{10pt}

\section{Derivation of Previous Complexity Bounds}
\label{app: previous results}
In this section, we briefly explain how to derive the sample complexities bounds in the first line of Table \ref{table: summary of results}.

In the most recent version of \cite{agarwal2019optimality}, a complexity bound of $\mathcal{O}(\varepsilon^{-6})$ can be obtained the taking $N = \mathcal{O}(\varepsilon^{-4})$ and $N = \mathcal{O}(\varepsilon^{-2})$ in its Corollary 6.2. Note this complexity bound can be improved to $\mathcal{O}(\varepsilon^{-4})$ if a uniform upper bound for exact NPG update directions is applied. In this case, one can apply the convergence bound of SGD instead of Projected SGD for the NPG subproblem. In this paper, we establish an upper bound for $\|\nabla J(\theta)\|$ in Lemma \ref{lem: smoothness of objective}. Therefore the exact NPG update direction is also upper bounded thanks to Assumption \ref{assump: strong convexity}. 

For \cite{wang2019neural}, the sample complexity bound of $\mathcal{O}(T_{TD}\varepsilon^{-2})$ is achieved by its Theorem 4.13. To be specific, one takes $T = \mathcal{O}(\varepsilon^{-2})$ and $T_{\text{TD}} = \mathcal{O}(m)$ number of temporal difference updates at each iteration. Here, $m$ is width of the neural network. 

Note that in the proof of its Corollary 4.14, we can choose $m = \mathcal{O}(T^4)$ (instead of $\mathcal{O}(T^6)$) to have a convergence bound of the form $\mathcal{O}(\sqrt{\varepsilon_0}) + \varepsilon$ (instead of $\mathcal{O}(\varepsilon)$), which is similar to our $\frac{\sqrt{\varepsilon_{\text{bias}}}}{1-\gamma} + \varepsilon$ convergence bound.

For \cite{liu2019neural}, by the Corollary 4.10 therein, one needs to take $K = \mathcal{O}(\varepsilon^{-2})$ and $T = \mathcal{O}(K^3) = \mathcal{O}(\varepsilon^{-6})$, which results in a total sample complexity of $\mathcal{O}(\varepsilon^{-8})$.

For \cite{shani2019adaptive}, its Theorem 5 (item 1) gives a sample complexity of $\sum_{k=1}^N M_k = \mathcal{O}(\varepsilon^{-4})$, where we have applied $N = \mathcal{O}(\varepsilon^{-2})$ and $M_k = \mathcal{O}(\varepsilon^{-2})$.

\section{Helper Lemmas}\label{sec:help_lemma} 

In this section, we lay out several results that will be useful in later analyses and proofs.

\subsection{Properties of PG Estimator}

First, for any $H>1$,  we define the $H$-horizon truncated versions of the return $J(\theta)$ as
\#\label{equ:def_truncated_return} J^H(\theta):=\EE_{s_0\sim\rho}\Bigg(\sum_{t=0}^{H-1} \gamma^t r_t\Bigg),
\#
where the expectation is taken over the trajectories, starting from the state distribution $\rho$.
Now we establish several properties of the GPOMDP policy gradient estimators and the  return functions. 

\begin{lemma}
\label{lem: smoothness of objective}
	Recall the GPOMDP policy gradient estimate given in \eqref{equ:GPOMDP_surro}. The following properties hold:
	\begin{itemize}
		\item If the infinite-sum in  \eqref{equ:GPOMDP_surro} is well defined, $g(\tau_i\given\theta)$ in \eqref{equ:GPOMDP_surro} is an unbiased estimate of the PG  $\nabla J(\theta)$. Similarly, the truncated GPOMDP estimate $g(\tau_i^H\given\theta)$ given by \eqref{equ: truncated GPOMDP estimator} is an unbiased estimate  of the PG $\nabla J^H(\theta)$. 
		\item $J(\theta), J^H(\theta)$ are $L_J$-smooth, where $L_J = \frac{MR}{(1-\gamma)^2} + \frac{2G^2 R}{(1-\gamma)^3}$. Furthermore, we have $\max\big\{\|\nabla {J}(\theta)\|,\|\nabla {J}^H(\theta)\|\big\}\leq \frac{ GR}{(1-\gamma)^2}$.
		\item We also have $\|\nabla J^H(\theta)-\nabla  J(\theta)\|\leq GR \left(\frac{H+1}{1-\gamma}+\frac{\gamma}{(1-\gamma)^2}\right)\gamma^H$.
	\end{itemize}
\end{lemma}
\begin{proof}
	 the unbiasedness of $g(\tau_i\given \theta)$ follows directly from \cite{baxter2001infinite}.  A similar decomposition can also be done for its truncated version $g(\tau_i^H\given \theta)$.
	
	
	The second argument follows directly from the Proposition 4.2 in \cite{xu2019samplearxiv}.
	
	For the third argument, one can calculate that
	\begin{align*}
	\|g(\tau_i^H\given \theta)- g(\tau_i\given \theta)\| & = \left\Vert\sum_{h=H}^{\infty} \left(\sum_{t=0}^{h}\nabla_{\theta}\log \pi_{\theta}(a^i_t\given s^i_t)\right)\left(\gamma^h r(s^i_h, a^i_h)\right)\right\Vert\\
	& \leq \|GR \sum_{h=H}^{\infty} (h+1) \gamma^h\|\\
	& = GR \left(\frac{H+1}{1-\gamma}+\frac{\gamma}{(1-\gamma)^2}\right)\gamma^H
	\end{align*}
	This rest of the proof follows from the unbiasedness of $g(\tau_i|\theta)$ and $g(\tau_i^H|\theta)$ for estimating $\nabla J(\theta)$ and $\nabla J^H(\theta)$, respectively. 
\end{proof}

\subsection{On the Positive Definiteness of $F_{\rho}(\theta)$}\label{sec:append_justify_Fisher}

Now we remark that the positive definiteness on the Fisher information matrix induced by $\pi_\theta$, as stated in Assumption \ref{assump: strong convexity}, is not  restricted. Assumption \ref{assump: strong convexity}  essentially states that $F(\theta)$ behaves well as a preconditioner in the NPG update \eqref{equ: NPG update_2}. 
This is a common (and minimal)  requirement for the convergence of preconditioned algorithms in both convex and nonconvex settings in the optimization realm \cite{byrd2016stochastic,moritz2016linearly, gower2016stochastic, wang2017stochastic, liu2019acceleration}. 

In the RL realm, one common example of policy parametrizations that can satisfy  this assumption is the Gaussian policy \cite{williams1992simple,duan2016benchmarking,papini2018stochastic,xu2019sample}, where $\pi_\theta(\cdot\given s)=\cN(\mu_\theta(s),\Sigma)$ with mean parametrized linearly as $\mu_\theta(s)=\phi(s)^\top \theta$, where $\phi(s)$ denotes some feature matrix of proper dimensions, $\theta$ is the coefficient vector, and $\Sigma\succ 0$ is some fixed covariance matrix. 
Suppose the action $a\in\cA\subseteq \RR^A$ and recall $\theta\in\RR^d$. Thus, $\phi(s)\in\RR^{d\times A}$.  
In this case, the Fisher information at each $s$ becomes $\phi(s)\Sigma^{-1}\phi(s)^{\top}$, independent of $\theta$, and is positive definite if $\phi(s)$ is full-row-rank. For the case $d<A$, which is usually the case as a lower-dimensional (than $a$) parameter $\theta$ is used, this can be achieved by designing the rows of $\phi(s)$ to be linearly independent,  a common requirement for linear function approximation settings  \cite{tsitsiklis1997analysis,melo2008analysis,sutton2009fast}.

For $\mu_\theta(s)$ being nonlinear functions of $\theta$, e.g., neural networks, the positive definiteness can still be satisfied, if the Jacobian of $\mu_\theta(s)$ at all $\theta$ uniformly satisfies the aforementioned conditions of $\phi(s)$ (the Jacobian in the linear case). 
In addition, beyond Gaussian policies, with the same conditions mentioned above on the feature $\phi(s)$ or the Jacobian of  $\mu_\theta(s)$, Assumption \ref{assump: strong convexity} also holds more generally for any full-rank exponential family parametrization with mean parametrized by $\mu_\theta(s)$, as the Fisher information matrix, in this case, is also positive definite, in replace of the covariance matrix $\phi(s)\Sigma^{-1}\phi(s)$ in the Gaussian case  \cite{dasgupta2011exponential}. 

Indeed, the Fisher information matrix is positive definite for any \emph{regular} statistical model \cite{kullback1997information}. In the pioneering NPG work \cite{kakade2002natural}, $F(\theta)$ is directly assumed to be positive definite. So is in the follow-up works on  natural actor-critic algorithms \cite{peters2008natural,bhatnagar2009natural}. 
In fact, this way, $F_{\rho}(\theta)$ will define a valid Riemannian metric on the  parameter space, which has been used for interpreting the desired convergence properties of natural gradient methods  \cite{amari1998natural,martens2014new}. In sum, the positive definiteness on the Fisher preconditioning matrix is common and not restrictive.

%
%
%
%

\section{SGD and Sampling Procedures} 
\label{app: sampling}
\subsection{SGD for Solving the Subproblems of NPG and SRVR-NPG}
Similar to the Algorithm 1 of \cite{agarwal2019optimality}, we also apply the averaged SGD algorithm as in \cite{bach2013non} to solve the subproblems of NPG and SRVR-NPG. 
\begin{procedure}[H]    
\caption{NPG-SGD}
\label{alg: SGD for NPG subproblem}
    \textbf{Input:} number of iterations $T$, stepsize $\alpha>0$, objective function $l(w)$, initialization $w_0=0$.
    \begin{algorithmic}[1]
        \For{$t\leftarrow 0,...,T-1$}{}
        \State{$w_{t+1} = w_t - \alpha \tilde{\nabla}l(w_t)$;}
        \Comment{$l(w)$ is defined in \eqref{equ: subproblem objective of NPG}, $\tilde{\nabla}l(w_t)$ is defined in \eqref{equ: subproblem stochastic gradient of NPG}.}
        \EndFor
        \State\Return{$w_{\text{out}} = \frac{1}{T}\sum_{t=1}^{T}w_t$.}
    \end{algorithmic}
\end{procedure}

\begin{procedure}[H]    
\caption{SRVR-NPG-SGD}
\label{alg: SGD for SRVR-NPG subproblem}
    \textbf{Input:} number of iterations $T$, stepsize $\alpha>0$, objective function $l(w)$, initialization $w_0=0$.
    \begin{algorithmic}[1]
        \For{$t\leftarrow 0,...,T-1$}{}
        \State{$w_{t+1} = w_t - \alpha \tilde{\nabla}l(w_t)$;}
        \Comment{$l(w)$ is defined in \eqref{equ: subproblem objective of SRVR-NPG}, $\tilde{\nabla}l(w_t)$ is defined in \eqref{equ: subproblem stochastic gradient of SRVR-NPG}.}
        \EndFor
        \State\Return{$w_{\text{out}} = \frac{1}{T}\sum_{t=1}^{T}w_t$.}
    \end{algorithmic}
\end{procedure}



For NPG, its subproblem \eqref{equ: NPG update_2} is of the form
\begin{align*}
w^k&\in \argmin_{w\in\Rd} L_{\nu^{\pi_{\theta^k}}}(w; \theta^k) =\EE_{(s,a)\sim\nu^{\pi_{\theta^k}}}\left[\big(A^{\pi_{\theta^k}}(s,a)-(1-\gamma)w^\top\nabla_{\theta}\log\pi_{\theta^k}(a\given s)\big)^2\right],
\end{align*}
where 
\[
\nu^{\pi_{\theta^k}}(s,a) = (1-\gamma)\EE_{(s_0,a_0)\sim\rho}\sum_{t=0}^\infty\gamma^t \PP(s_t=s,a_t=a\given s_0,a_0,\pi_{\theta^k}).
\]
In Procedure \ref{alg: SGD for NPG subproblem}, let us set 
\begin{align}
\label{equ: subproblem objective of NPG}
l(w) = \frac{1}{2(1-\gamma)^2}L_{\nu^{\pi_{\theta^k}}}(w; \theta^k). 
\end{align}
Then, we can obtain a stochastic gradient at $w_t$ by
\begin{align}
\label{equ: subproblem stochastic gradient of NPG}
\begin{split}
\tilde{\nabla} l(w_t) &=\left((w_{t})^T\nabla_\theta \log\pi_{\theta^k}(a|s)-\frac{1}{1-\gamma}\hat{A}^{\pi_{\theta^k}}(s,a)\right)\nabla_{\theta}\log\pi_{\theta^k}(a|s)\\
\end{split}
\end{align}
where $(s,a)\sim \nu^{\pi_{\theta^k}}$, and $\hat{A}^{\pi_{\theta^k}}(s,a)$ is an unbiased estimate of $A^{\pi_{\theta^k}}(s,a)$. We will describe how to obtain $(s,a)\sim \nu^{\pi_{\theta^k}}$ and $\hat{A}^{\pi_{\theta^k}}(s,a)$ in App. \ref{app: sampling details}.

Following Corollary 6.10 of \cite{agarwal2019optimality}, we can verify that $\tilde{\nabla}l(w_t)$ is an unbiased estimate of $\nabla {l}(w_t)$.

For SRVR-NPG, its subproblem \eqref{equ: SRVR-NPG subproblem} is of the form 
\begin{align*}
w^{j+1}_{t}\approx \argmin_{w}\left\{\underset{(s,a)\sim \nu^{\pi_{\theta^{j+1}_t}}}{\mathbb{E}}\left[\big( w^T\nabla_{\theta}\log \pi_{\theta^{j+1}_t}(a\given s)\big)^2\right]-2\langle  w, u^{j+1}_t\rangle\right\},
\end{align*}
where 
\[
\nu^{\pi_{\theta^{j+1}_t}}(s,a) = (1-\gamma)\EE_{(s_0,a_0)\sim\rho}\sum_{t=0}^\infty\gamma^t \PP(s_t=s,a_t=a\given s_0,a_0,\pi_{\theta^{j+1}_t}).
\]
In Procedure \ref{alg: SGD for SRVR-NPG subproblem}, let us set
\begin{align}
\label{equ: subproblem objective of SRVR-NPG}
l(w) = \frac{1}{2}\left(\underset{(s,a)\sim \nu^{\pi_{\theta^{j+1}_t}}}{\mathbb{E}}\left[\big(w^T\nabla_{\theta}\log \pi_{\theta^{j+1}_t}(a\given s)\big)^2\right]-2\langle  w, u^{j+1}_t\rangle\right).
\end{align}
Then, a stochastic gradient $\tilde{\nabla} l (w_t)$ is given by
\begin{align}
\label{equ: subproblem stochastic gradient of SRVR-NPG}
\tilde{\nabla} l(w_t) =\left((w_{t})^T\nabla_\theta \log\pi_{\theta^{j+1}_t}(a|s)\right)\nabla_{\theta}\log\pi_{\theta^{j+1}_t}(a|s)-u^{j+1}_t.
\end{align}
where $(s,a)\sim \nu^{\pi_{\theta^{j+1}_t}}$ is obtained in a similar way as above. It is straightforward to verify that $\tilde{\nabla}l(w_t)$ is an unbiased estimate of $\nabla {l}(w_t)$. 
\subsection{Sampling Procedures}
\label{app: sampling details}
Sampling $(s,a)\sim \nu^{\pi_{\theta^k}}$ and  Obtaining $\hat{A}^{\pi_{\theta^k}}(s,a)$ can be done in a standard way, for example, by apply Algorithm 3 of \cite{agarwal2019optimality}. Both of them needs to sample $\frac{1}{1-\gamma}$ state-action pairs in expectation.
%
%
%
%



\section{SRVR-PG Algorithm}
\label{app: SRVR-PG algorithm}

The Stochastic Recursive Variance-Reduced PG (SRVR-PG) algorithm is introduced in \cite{xu2019sample}, where a recursively updated semi-stochastic gradient $u^{j+1}_t$ is applied as an update direction. 
\begin{algorithm}[H]    
\caption{Stochastic Recursive Variance Reduced Policy Gradient (SRVR-PG)}
\label{alg: SRVR-PG}
    \textbf{Input:} number of epochs $S$, epoch size $m$, stepsize $\eta$, batch size $N$, minibatch size $B$, truncation horizon $H$, initial parameter $\theta^0_m=\theta_0\in\Rd.$
    \begin{algorithmic}[1]
        \For{$j\leftarrow 0,...,S-1$}{}
        \State{$\theta^{j+1}_0=\theta^j_m$;}
        \State{Sample $N$ trajectories $\{\tau^H_i\}_{i=1}^N$ from $p^H_{\rho}(\cdot|\theta^{j+1}_0)$;} 
        \State{$u^{j+1}_0=\frac{1}{N}\sum_{i=1}^N g(\tau^H_i|\theta^{j+1}_0)$;}
        \State{$\theta^{j+1}_1 = \theta^{j+1}_0 -\eta \nu^{j+1}_0$;}
        \For{$t \leftarrow 1,...,m-1$}{}
        \State{Sample $B$ trajectories $\{\tau^H_j\}_{j=1}^{B}$ from $p^H_{\rho}(\cdot|\theta^{j+1}_t)$;} 
        \State{$u^{j+1}_t=u^{j+1}_{t-1}+\frac{1}{B}\sum_{j=1}^B \left(g(\tau^H_j|\theta^{j+1}_t)-g_w(\tau^H_j|\theta^{j+1}_{t-1})\right)$;}
        \State{$\theta^{j+1}_{t+1} = \theta^{j+1}_t - \eta u^{j+1}_t$;}
        \EndFor
        \EndFor
        \State\Return{$\theta_{\text{out}}$ chosen uniformly from $\{\theta\}_{j=1,...,S; t=0,...,m-1.}$}
    \end{algorithmic}
\end{algorithm}
Here, the gradient estimators $g$ and $g_w$ are defined in \eqref{equ: truncated GPOMDP estimator} and \eqref{equ: weighted GPOMDP}, respectively.

\section{Stationary Convergence}
\label{sec: stationary convergence}
%
%

In this section, we proceed to establish the stationary convergence of stochastic PG, NPG, SRVR-PG, and SRVR-NPG from an optimization perspective. 

The stationary convergence of stochastic PG follows from the analysis of SGD. For SRVR-PG, we adapt its analysis in \cite{xu2017stochastic}. 

For NPG and SRVR-NPG, the Fisher information matrix $F(\theta)$ is applied as a preconditioner on top of PG and SRVR-PG, respectively. Regarding $F(\theta)$, we know from Assumptions \ref{assump: strong convexity} and \ref{assump: conditions on score function} that
\[
\mu_F I_d \preccurlyeq  F(\theta) \preccurlyeq G^2 I_d\,\,\,\text{for any}\,\,\,\theta \in \Rd.
\]
Since $\mu_F>0$, we know that $F(\theta)$ defines a nice metric around $\theta$. Consequently, with the analysis of gradient methods in nonconvex optimization, one can show that NPG (SRVR-NPG) has a similar iteration complexity compared with PG (SRVR-PG), although at each iteration, a subproblem needs to be solved in order to obtain an approximate preconditioned update direction.

We next present the stationary convergence results, and prove them in the subsequent sections. These results are established for $J^H(\theta)$ or $J(\theta)$, and we will apply the intermediate results in their proof to establish the global convergence on $J(\theta)$ (up to function approximation errors due to policy parametrizations).

\begin{theorem}
\label{thm: PG stationary convergence}
In the stochastic PG update \eqref{equ: PG update}, by choosing $\eta=\frac{1}{4 L_J}$, $K=\frac{32L_J(J^{H,\star}-J^H(\theta_0))}{\varepsilon},$ and $N = \frac{6\sigma^2}{\varepsilon}$, we have
    \begin{align*}
\frac{1}{K}\sum_{k=0}^{K-1} \mathbb{E}[\|\nabla J^H(\theta^{k})\|^2]\leq \varepsilon.
\end{align*}
In total, stochastic PG samples $\mathcal{O}\left(\frac{\sigma^2}{(1-\gamma)^2\varepsilon^2}\right)$ trajectories.
\end{theorem}
\begin{theorem}
\label{thm: NPG stationary convergence}
In the NPG update \eqref{equ: NPG update_2}, let us apply $\mathcal{O}\left(\frac{1}{(1-\gamma)^4\varepsilon}\right)$ iterations of SGD as in Procedure \ref{alg: SGD for NPG subproblem} to obtain an update direction $w^k$. 
In addition, let us take $\eta = \frac{\mu_F^2}{4G^2L_J}$ and $K=\frac{32L_JG^4(J^{\star}- J(\theta_0)}{\mu_F^2\varepsilon}$. Then, we have
\begin{align*}
\frac{1}{K}\sum_{k=0}^{K-1} \mathbb{E}[\|\nabla J(\theta^{k})\|^2]
&\leq \varepsilon.
\end{align*}
In total, NPG samples $\mathcal{O}\left(\frac{1}{(1-\gamma)^6\varepsilon^2}\right)$
trajectories.

\end{theorem}

\begin{corollary}(Theorem 4.5 of \cite{xu2019sample})
\label{thm: SRVR-PG stationary convergence}
In SRVR-PG (Algorithm \ref{alg: SRVR-PG}), take $\eta = \frac{1}{4L_J}$, 
$N = \frac{12\sigma^2}{\varepsilon}$, $S = \frac{64MR (J^{\star}-J(\theta^0))}{(1-\gamma)^{2.5}\varepsilon^{0.5}}$, $m = \frac{(1-\gamma)^{0.5}}{\varepsilon^{0.5}}$, and $B = \frac{72\eta G^2(2G^2+M)(W+1)\gamma}{M(1-\gamma)^3}m$. Then, we have
\[
\frac{1}{Sm}\sum_{s=0}^{S-1}\sum_{t=0}^{m-1}\EE\|\nabla J^H(\theta^{j+1}_t)\|^2 \leq \varepsilon.
\]
In total, SRVR-PG samples $\mathcal{O}\left(\frac{W+\sigma^2}{(1-\gamma)^{2.5}\varepsilon^{1.5}}\right)$  trajectories.
\end{corollary}

\begin{theorem}
\label{thm: SRVR-NPG stationary convergence}
In SRVR-NPG (Algorithm \ref{alg: SRVR-NPG}), take $\eta = \frac{\mu_F}{8L_J}$, $S = \frac{24G^2(J^{H,\star}- J^H(\theta_0))}{\eta \varepsilon^{0.5}}$, $m = \frac{1}{\varepsilon^{0.5}}$, $B = \left(\frac{\eta}{\mu_F}+\frac{\eta}{4G^2}\right)\frac{72R G^2 (2G^2+M)(W+1)\gamma}{(1-\gamma)^5}\frac{1}{L_J\varepsilon^{0.75}}$, and $ 
N = 3\left(\frac{8G^2}{\mu_F}+2\right)\frac{\sigma^2}{\varepsilon}.$ 
In addition, assume that $\varepsilon$ is small enough such that
\begin{align*}
\varepsilon\leq \min\Big\{ & 3\left(\frac{8G^2}{\mu_F}+2\right)\left(\frac{GR}{(1-\gamma)^2}\right)^2, 3\left(\frac{8G^2}{4}+\frac{8G^4}{4\mu_F}\right)\frac{2}{\mu_F}\left(\frac{GR}{(1-\gamma)^2}\right)^2, \\
& \left(\frac{2}{3\eta L_J}(\mu_F+\frac{\mu_F^2}{4G^2})\right)^4 \Big\}.
\end{align*}
Let us also apply $\mathcal{O}\left(\frac{1}{(1-\gamma)^4\varepsilon}\right)$ iterations of SGD as in Procedure \ref{alg: SGD for SRVR-NPG subproblem} to obtain an update direction $w^{j+1}_t$.
Then, in order to have
\[
\frac{1}{Sm}\sum_{s=0}^{S-1}\sum_{t=0}^{m-1}\EE\|\nabla J^H(\theta^{j+1}_t)\|^2 \leq \varepsilon,
\]
SRVR-NPG samples $\cO\left(\frac{\sigma^2}{(1-\gamma)^2\varepsilon^{1.5}}+\frac{W}{(1-\gamma)^{3}\varepsilon^{1.75}}+\frac{1}{(1-\gamma)^6\varepsilon^{2}}\right)$  trajectories.
\end{theorem}

\section{Proof of Theorem \ref{thm: PG stationary convergence}}
\label{app: PG stationary}


\begin{proof}
Let $g^k = \frac{1}{N}\sum_{i=1}^N g(\tau^H_i|\theta^k)$. Then, we have
\begin{align}
\label{equ: V ascent PG}
\begin{split}
 J^H(\theta^{k+1})&\geq  J^H(\theta^{k})+\langle \nabla  J^H(\theta^k), \theta^{k+1}-\theta^k\rangle-\frac{L_J}{2}\|\theta^{k+1}-\theta^k\|^2\\
& =  J^H(\theta^{k})+\eta\langle \nabla  J^H(\theta^k), g^k\rangle-\frac{L_J\eta^2}{2}\|g^k\|^2\\
& =  J^H(\theta^{k})+\eta\langle \nabla  J^H(\theta^k), g^k- \nabla  J^H(\theta^k)+\nabla  J^H(\theta^k) \rangle\\
&\,\,\, -\frac{L_J\eta^2}{2}\|g^k-\nabla  J^H(\theta^k)+\nabla  J^H(\theta^k)\|^2\\
&\geq  J^H(\theta^{k})+\frac{\eta}{2}\|\nabla  J^H(\theta^{k})\|^2-\frac{\eta}{2}\|g^k-\nabla  J^H(\theta^{k})\|^2\\
&\,\,\, -L_J\eta^2\|g^k-\nabla  J^H(\theta^{k})\|^2-L_J\eta^2\|\nabla  J^H(\theta^{k})\|^2\\
&=  J^H(\theta^{k})+(\frac{\eta}{2}-L_J\eta^2)\|\nabla  J^H(\theta^{k})\|^2-(\frac{\eta}{2}+L_J\eta^2)\|g^k-\nabla  J^H(\theta^{k})\|^2,
\end{split}
\end{align}
where we have applied Lemma \ref{lem: smoothness of objective} in the first inequality, and Cauchy-Schwartz in the second inequality. 

Taking expectation on both sides and applying Lemma \ref{lem: smoothness of objective} and Assumption \ref{assump: variance} yields
\begin{align*}
	\mathbb{E} [ J^H(\theta^{k+1})]&\geq \mathbb{E} [ J^H(\theta^{k})]+(\frac{\eta}{2}-L_J\eta^2)\mathbb{E}[\|\nabla  J^H(\theta^{k})\|^2]-(\frac{\eta}{2}+L_J\eta^2)\frac{\sigma^2}{N}.
\end{align*}
Let us further telescope from $k=0$ to $K-1$ to obtain
\begin{align}
\label{equ: PG stationary final}
\frac{1}{K}\sum_{k=0}^{K-1} \mathbb{E}[\|\nabla  J^H(\theta^{k})\|^2]\leq \frac{\frac{J^{H, \star}- J^H(\theta_0)}{K}+(\frac{\eta}{2}+L_J\eta^2)\frac{\sigma^2}{N}}{\frac{\eta}{2}-L_J\eta^2}.
\end{align}
Taking $\eta=\frac{1}{4 L_J}$, $K=\frac{32L_J(J^{\star}- J^H(\theta_0))}{\varepsilon},$ and $N = \frac{6\sigma^2}{\varepsilon}$ gives 
\begin{align*}
\frac{1}{K}\sum_{k=0}^{K-1} \mathbb{E}[\|\nabla  J^H(\theta^{k})\|^2]\leq \varepsilon.
\end{align*}
Finally, by applying $L_J=\frac{MR}{(1-\gamma)^2} + \frac{2G^2R}{(1-\gamma)^3}$, we know that PG needs to sample $KN=\mathcal{O}\left(\frac{\sigma^2}{(1-\gamma)^2\varepsilon^2}\right)$ trajectories.
\end{proof}

\section{Proof of Theorem \ref{thm: NPG stationary convergence}}
\label{app: NPG stationary}
Before proving Theorem \ref{thm: NPG stationary convergence}, let us first establish the sample complexity of SGD when applied to obtain an approximate NPG update direction $w^k$. 
\begin{proposition}
\label{prop: SGD convergence NPG}
In Procedure \ref{alg: SGD for NPG subproblem}, take $\alpha = \frac{1}{4G^2}$ and let the objective be 
\[
l(w) = 	\frac{1}{2(1-\gamma)^2} L_{\nu^{\pi_{\theta}}}(w; \theta^k)=\frac{1}{2}\EE_{(s,a)\sim\nu^{\pi_{\theta^k}}}\big[\frac{1}{1-\gamma}A^{\pi_{\theta^k}}(s,a)- w^\top\nabla_{\theta}\log\pi_{\theta^k}(a\given s)\big]^2.
\] 
Let $w^k_{\star}$ be the minimizer of $l(w)$. Then, in order to achieve
\begin{align*}
\mathbb{E}[\|w_{\text{out}}-w^k_{\star}\|^2]\leq \varepsilon',
\end{align*}
Procedure \ref{alg: SGD for NPG subproblem} requires sampling 
\[
\frac{4\left(\frac{2 G^2R+4\mu_F R}{\mu_F(1-\gamma)^2}\sqrt{d}+\frac{G^2R}{\mu_F(1-\gamma)^2}\right)^2}{\mu_F\varepsilon'} = \cO\left(\frac{1}{(1-\gamma)^4\varepsilon'}\right)
\] 
trajectories. 
\end{proposition}
\begin{proof}
In this proof, we will suppress the superscript $k$. 

Let $l^{\star}=\min_{w\in\Rd}l(w)$ and $w_{\star} = \argmin_{w\in\Rd} l(w)$.

By Theorem 1 of \cite{bach2013non}, we know that
\[
\EE[l(w_{\text{out}})-l^{\star}]\leq \frac{2(\xi\sqrt{d}+G\EE[\|w_{\star}\|])^2}{T},
\] 
where $l^{\star}$ is the minimum of $l(w)$, and $\xi$ is defined such that
\[
\EE [g_{\star}{(g_{\star})^T}] \preccurlyeq \xi^2 \nabla^2_{w} l(w),
\]
where $g_{\star}$ is a stochastic gradient of $l(w)$ at $w_{\star}$.

%

Following the proof of Corollary 6.10 of \cite{agarwal2019optimality}(arXiv V2 version), we obtain an upper bound of $\xi$ as follows.

By \eqref{equ: subproblem stochastic gradient of NPG} we know that
\[
g_{\star} =\left((w_{\star})^T\nabla_\theta \log\pi_{\theta}(a|s)-\frac{1}{1-\gamma}\hat{A}^{\pi_\theta}(s,a)\right)\nabla_{\theta}\log\pi_{\theta}(a|s),
\]
where $s,a\sim\nu^{\pi_{\theta}}$ and $a'\sim \nu^{\pi_{\theta}}(\cdot|s)$. 
 Therefore, we can stipulate that 
 \begin{align*}
 \EE \left[\left(w_\star^T \nabla_\theta \log\pi_{\theta}(a|s) - \frac{1}{1-\gamma}\hat{A}^{\pi_\theta}(s,a)\right)^2 | (s, a)\right] \leq \xi^2
 \end{align*}
From Lemma \ref{lem: smoothness of objective} and Assumption \ref{assump: strong convexity} we have
\[
\|w_{\star}\| = \|F_{\rho}^{-1}(\theta)\nabla J(\theta)\|\leq \frac{GR}{\mu_F(1-\gamma)^2}.
\]
Therefore, it suffices to have
\begin{align*}
 &\EE \left[\left(w_\star^T \nabla_\theta \log\pi_{\theta}(a|s) - \frac{1}{1-\gamma}\hat{A}^{\pi_\theta}(s,a)\right)^2| (s, a)\right] \\
 &\leq 2 \EE \left[\left(w_\star^T \nabla_\theta \log\pi_{\theta}(a|s)\right)^2 | (s, a) \right] + \frac{2}{(1-\gamma)^2} \EE \left[\left(\hat{A}^{\pi_\theta}(s,a)\right)^2 | (s, a) \right]\\
 &\leq 2\left(\frac{GR}{\mu_F(1-\gamma)^2}\right)^2 G^2 + 	\frac{2}{(1-\gamma)^2} \EE \left[\left(\hat{A}^{\pi_\theta}(s,a)\right)^2 | (s, a) \right] \\
& \leq \xi^2
\end{align*}
Since
\[
\EE \left[\left(\hat{A}^{\pi_\theta}(s,a)\right)^2 | (s, a) \right] \leq 2 \EE \left[\left(\hat{Q}^{\pi_\theta}(s,a) | (s, a) \right)^2\right]  + 2 \EE \left[\left(\hat{V}^{\pi_\theta}(s,a)\right)^2 | (s, a) \right] \leq \frac{8R^2}{(1-\gamma)^2}
\]
We can safely take 
\[
\xi = \frac{2 G^2R+4\mu_F R}{\mu_F(1-\gamma)^2}
\]
This leads to
\[
\EE[l(w_{\text{out}}) - l(w_{\star})]\leq \frac{2\left(\frac{2 G^2R+4\mu_F R}{\mu_F(1-\gamma)^2}\sqrt{d}+\frac{G^2R}{\mu_F(1-\gamma)^2}\right)^2}{T}.
\] 
Since $l(w)$ is $\mu_F-$strongly convex, in order to achieve $\mathbb{E}[\|w_{\text{out}}-w^k_{\star}\|^2]\leq \varepsilon'$, let us set
\begin{align*}
\EE[l(w_{\text{out}}) - l(w_{\star})]\leq  \frac{\mu_F}{2}\varepsilon'.
\end{align*}
Then, we need 
\[
T=\frac{4\left(\frac{2 G^2R+4\mu_F R}{\mu_F(1-\gamma)^2}\sqrt{d}+\frac{G^2R}{\mu_F(1-\gamma)^2}\right)^2}{\mu_F\varepsilon'} = \cO\left(\frac{1}{(1-\gamma)^4\varepsilon'}\right).
\]
Since each stochastic gradient of SGD has a cost of $\frac{2}{1-\gamma}$ (see App. \ref{app: sampling}), this means to sample $\cO\left(\frac{1}{(1-\gamma)^4\varepsilon'}\right)$ trajectories.
\end{proof}

Now, we are ready to prove Theorem \ref{thm: NPG stationary convergence}.
\begin{proof}[Proof of Theorem \ref{thm: NPG stationary convergence}]
We apply SGD to obtain a $w^k$ such that
\begin{align}
\label{equ: NPG subproblem accuracy}
\mathbb{E}\|w^k- F^{-1}(\theta^k)\nabla  J(\theta^k)\|^2\leq \frac{\mu_F^2\varepsilon}{32\eta^2G^4L^2_J\left(\frac{2G^4}{\mu^2_F}+1\right)}=\cO\left(\varepsilon\right).
\end{align}
By Proposition \ref{prop: SGD convergence NPG}, we need 
to sample $\mathcal{O}\left(\frac{1}{(1-\gamma)^4\varepsilon}\right)$ trajectories. 

From \eqref{equ: NPG subproblem accuracy} we have
\begin{align}
\label{equ: point accuracy NPG}
\mathbb{E}\|\theta^{k+1}- \theta^{k+1}_{\star}\|^2=\eta^2\mathbb{E}\|w^k- F^{-1}(\theta^k)\nabla  J(\theta^k)\|^2\leq \frac{\mu_F^2\varepsilon}{32G^4L^2_J\left(\frac{2G^4}{\mu^2_F}+1\right)},
\end{align}
where $\theta^{k+1}_{\star}= \theta^k+\eta F^{-1}(\theta^k)\nabla  J(\theta^k)$. 

By Lemma \ref{lem: smoothness of objective} and Assumption \ref{assump: conditions on score function} we have 
\begin{align*}
 J(\theta^{k+1})&\geq  J(\theta^k)+\langle \nabla  J(\theta^k), \theta^{k+1}_{\star}-\theta^k\rangle +\langle \nabla  J(\theta^k), \theta^{k+1}-\theta^{k+1}_{\star}\rangle -\frac{L_J}{2}\|\theta^{k+1}-\theta^k\|^2\\
&= J(\theta^k)+\eta \langle \nabla  J(\theta^k), F^{-1}(\theta^k)\nabla  J(\theta^k)\rangle \\
&\,\,\, +\langle \nabla  J(\theta^k), \theta^{k+1}-\theta^{k+1}_{\star}\rangle -\frac{L_J}{2}\|\theta^{k+1}-\theta^k\|^2\\
&\geq  J(\theta^k) +\frac{\eta}{G^2}\|\nabla  J(\theta^k)\|^2 +\langle \nabla  J(\theta^k), \theta^{k+1}-\theta^{k+1}_{\star}\rangle -\frac{L_J}{2}\|\theta^{k+1}-\theta^k\|^2.
\end{align*}
Therefore, 
\begin{align*}
 J(\theta^{k+1})&\geq  J(\theta^k) +\frac{\eta}{2G^2}\|\nabla  J(\theta^k)\|^2 -\frac{G^2}{2\eta}\|\theta^{k+1}-\theta^{k+1}_{\star}\|^2 -\frac{L_J}{2}\|\theta^{k+1}-\theta^k\|^2\\
&\geq  J(\theta^k) +\frac{\eta}{2G^2}\|\nabla  J(\theta^k)\|^2 -\left(\frac{G^2}{2\eta}+L_J\right)\|\theta^{k+1}-\theta^{k+1}_{\star}\|^2 -L_J\|\theta^{k+1}_{\star}-\theta^k\|^2\\
&\geq  J(\theta^k) +\left(\frac{\eta}{2G^2}-\frac{L_J\eta^2}{\mu_F^2}\right)\|\nabla  J(\theta^k)\|^2 -\left(\frac{G^2}{2\eta}+L_J\right)\|\theta^{k+1}-\theta^{k+1}_{\star}\|^2,
\end{align*}
where we have applied Cauchy-Schwartz in the first and second inequalities, and $\theta^{k+1}_{\star}= \theta^k+\eta F^{-1}(\theta^k)\nabla  J(\theta^k)$ in the last step.  

Taking full expectation on both sides yields
\begin{align*}
\EE[ J(\theta^{k+1})]&\geq \EE[ J(\theta^k)] +\left(\frac{\eta}{2G^2}-\frac{L_J\eta^2}{\mu_F^2}\right)\EE\|\nabla  J(\theta^k)\|^2 -\left(\frac{G^2}{2\eta}+L_J\right)\EE\|\theta^{k+1}-\theta^{k+1}_{\star}\|^2\\
&\geq \EE[ J(\theta^k)] +\left(\frac{\eta}{2G^2}-\frac{L_J\eta^2}{\mu_F^2}\right)\EE\|\nabla  J(\theta^k)\|^2 - \left(\frac{G^2}{2\eta}+L_J\right)\frac{\mu_F^2\varepsilon}{32G^4L^2_J\left(\frac{2G^4}{\mu^2_F}+1\right)},
\end{align*}
where we have applied \eqref{equ: point accuracy NPG} in the second inequality.


Telescoping the above inequality from $k=0$ to $k=K-1$ gives
\begin{align*}
\frac{ J^{\star}- J(\theta_0)}{K}&\geq \left(\frac{\eta}{2G^2}-\frac{L_J\eta^2}{\mu_F^2}\right)\frac{1}{K}\sum_{k=0}^{K-1}\mathbb{E}\|\nabla  J(\theta^k)\|^2 - \left(\frac{G^2}{2\eta}+L_J\right)\frac{\mu_F^2\varepsilon}{32G^4L^2_J\left(\frac{2G^4}{\mu^2_F}+1\right)}.
\end{align*}
Finally, by taking $\eta = \frac{\mu_F^2}{4G^2L_J}$ and $K=\frac{32L_JG^4( J^{\star}- J(\theta_0)}{\mu_F^2\varepsilon}=\cO\left(\frac{1}{(1-\gamma)^2\varepsilon}\right)$, we arrive at
\begin{align}
\label{equ: NPG stationary final}
\frac{1}{K}\sum_{k=0}^{K-1} \mathbb{E}[\|\nabla  J(\theta^{k})\|^2]
&\leq \frac{\frac{ J^{\star}- J(\theta_0)}{K} + \frac{\left(\frac{G^2}{2\eta}+L_J\right)\mu_F^2\varepsilon}{32G^4L^2_J\left(\frac{2G^4}{\mu^2_F}+1\right)}}{\left(\frac{\eta}{2G^2}-\frac{L_J\eta^2}{\mu_F^2}\right)}=\varepsilon.
\end{align}
Recall that at each iteration of NPG, we apply SGD as in Procedure \ref{alg: SGD for NPG subproblem} to reach \eqref{equ: NPG subproblem accuracy}. By Proposition \ref{prop: SGD convergence NPG}, we know that in total, NPG requires to sample 
\[
\frac{32L_JG^4( J^{\star}- J(\theta_0))}{\mu_F^2\varepsilon}\cdot\frac{4\left(\frac{2 G^2R+4\mu_F R}{\mu_F(1-\gamma)^2}\sqrt{d}+\frac{G^2R}{\mu_F(1-\gamma)^2}\right)^2}{\mu_F\frac{\mu_F^2\varepsilon}{32\eta^2G^4L^2_J\left(\frac{2G^4}{\mu^2_F}+1\right)}} = \cO\left(\frac{1}{(1-\gamma)^6\varepsilon^2}\right)
\]
trajectories. 

\end{proof}

%

\section{Proof of Theorem \ref{thm: SRVR-PG stationary convergence}}

\begin{proof}
By Theorem 4.5 of \cite{xu2019sample}, we know that if $\eta = \frac{1}{4L_J}$ and
\[
B= \frac{3\eta C_{\gamma}m}{L_J}=\frac{72\eta G^2(2G^2+M)(W+1)\gamma}{M(1-\gamma)^3}m,
\]
then
\[
\frac{1}{Sm}\sum_{s=0}^{S-1}\sum_{t=0}^{m-1}\EE\|\nabla  J^H(\theta^{j+1}_t)\|^2 \leq \frac{8(J^{\star}- J^H(\theta_0))}{\eta Sm}+\frac{6\sigma^2}{N}. 
\]
Therefore, taking $N = \frac{12\sigma^2}{\varepsilon}$ and $Sm = \frac{64MR ( J^{\star}- J^H(\theta^0))}{(1-\gamma)^2\varepsilon}$
yields 
\[
\frac{1}{Sm}\sum_{s=0}^{S-1}\sum_{t=0}^{m-1}\EE\|\nabla  J^H(\theta^{j+1}_t)\|^2 \leq \varepsilon.
\]
Let us take $S = \frac{64MR ( J^{\star}- J^H(\theta^0))}{(1-\gamma)^{2.5}\varepsilon^{0.5}}$ and $m = \frac{(1-\gamma)^{0.5}}{\varepsilon^{0.5}}$. Then, the number of trajectories required by SRVR-PG is
\begin{align*}
S(N+mB) &= S \frac{12\sigma^2}{\varepsilon} + \frac{64MR ( J^{\star}- J^H(\theta^0))}{(1-\gamma)^2\varepsilon} B \\
&= S \frac{12\sigma^2}{\varepsilon} + \frac{64MR ( J^{\star}- J^H(\theta^0))}{(1-\gamma)^2\varepsilon} \frac{72\eta G^2(2G^2+M)(W+1)\gamma}{M(1-\gamma)^3}m\\
& = \cO\left(\frac{\sigma^2}{(1-\gamma)^{2.5}\varepsilon^{1.5}}+\frac{W}{(1-\gamma)^{2.5}\varepsilon^{1.5}}\right)\\
& = \cO\left(\frac{W+\sigma^2}{(1-\gamma)^{2.5}\varepsilon^{1.5}}\right).
\end{align*}
Therefore, SRVR-PG needs to sample $\cO\left(\frac{W+\sigma^2}{(1-\gamma)^{2.5}\varepsilon^{1.5}}\right)$ trajectories.
\end{proof}

\section{Proof of Theorem \ref{thm: SRVR-NPG stationary convergence}}
\label{app: SRVR-NPG stationary}

In order to prove Theorem \ref{thm: SRVR-NPG stationary convergence}, we need the following technical results.

\begin{lemma}[Equation B.10 of \cite{xu2019improved}]
\label{lem: nabla J and v}
We have
\begin{align*}
\mathbb{E}\|\nabla  J^H(\theta^{j+1}_t) - u^{j+1}_t\|^2\leq \frac{C_{\gamma}}{B}\sum_{l=1}^t \mathbb{E}\|\theta^{j+1}_{l}-\theta^{j+1}_{l-1}\|^2+\frac{\sigma^2}{N},
\end{align*}
where 
\[
C_\gamma = \frac{24R G^2 (2G^2+M)(W+1)\gamma}{(1-\gamma)^5}.
\]
\end{lemma}
\begin{proof}
This lemma is adapted from the Equation B.10 of \cite{xu2019improved}, where SRVR-PG is analyzed. It is also true for our SRVR-NPG since the update rule of $u^{j+1}_t$ is the same for both algorithms. 
\end{proof}

\begin{proposition}
\label{prop: SGD convergence SRVR-NPG}
In SRVR-NPG, apply SGD as in Procedure \ref{alg: SGD for SRVR-NPG subproblem} to solve the subproblems.  Take $\alpha = \frac{1}{4G^2}$ and let the objective be 
\[
l(w) =\frac{1}{2}\left(\mathbb{E}_{(s,a)\sim \nu^{j+1}_t}[w^T\nabla_{\theta}\log \pi_{\theta^{j+1}_t}(a\given s)]^2-2\langle \eta w, u^{j+1}_t\rangle\right).
\] 
Let $w^{j+1}_{t,\star} = F^{-1}(\theta^{j+1}_t)\nabla  J^H(\theta^{j+1}_t)$ be the minimizer of $l(w)$. Assume in addition that
\begin{align*}
\frac{\sigma^2}{N}&\leq  \left(\frac{GR}{(1-\gamma)^2}\right)^2,\\
\varepsilon' &\leq  \frac{2}{\mu_F^2}\left(\frac{GR}{(1-\gamma)^2}\right)^2\\
\frac{C_\gamma m}{B}2\eta^2&\leq \frac{1}{3}\mu_F^2. 
\end{align*}

Then, in order to achieve 
\begin{align*}
\mathbb{E}[\|w^{j+1}_t-w^{j+1}_{t,\star})\|^2]\leq \varepsilon'
\end{align*}
for each $s=0,1,...,S-1$ and $t = 0,1,...,m-1$, 
Procedure \ref{alg: SGD for SRVR-NPG subproblem} requires sampling 
\[
\frac{4\left(\frac{\frac{{2}}{\mu_F}\frac{GR}{(1-\gamma)^2}G^2+\frac{{2}GR}{(1-\gamma)^2}}{\sqrt{\mu_F}}\sqrt{d}+\frac{{2}G^2R}{\mu_F(1-\gamma)^2}\right)^2}{\mu_F\varepsilon'} = \cO\left(\frac{1}{(1-\gamma)^4\varepsilon'}\right)
\]
trajectories. 
\end{proposition}

\begin{proof}[Proof of Proposition \ref{prop: SGD convergence SRVR-NPG}]


Recall that we are applying SGD as in Procedure \ref{alg: SGD for SRVR-NPG subproblem} to solve the SRVR-NPG subproblem \eqref{equ: SRVR-NPG subproblem}.

Let us focus on $t=0$, where $u^{j+1}_0 = \frac{1}{N}\sum_{i=1}^N g(\tau^H_i|\theta^{j+1}_0)$. As a result, $\EE[u^{j+1}_0]=\nabla  J^H(\theta^{j+1}_0)$ and $\text{Var}(u^{j+1}_0)\leq \frac{\sigma^2}{N}$. Therefore,
\begin{align*}
\EE[\|w^{j+1}_{0,\star}\|^2]&=\EE[\|F^{-1}(\theta^{j+1}_0)u^{j+1}_0\|^2]\\
&\leq \frac{1}{\mu_F^2}\EE[\|u^{j+1}_0\|]\leq \frac{1}{\mu_F^2}\left(\frac{GR}{(1-\gamma)^2}\right)^2+ \frac{1}{\mu_F^2}\frac{\sigma^2}{N}\leq \frac{4}{\mu_F^2}\left(\frac{GR}{(1-\gamma)^2}\right)^2.
\end{align*}
Recall from \eqref{equ: subproblem stochastic gradient of SRVR-NPG} that a stochastic gradient $\nabla l (w_t)$ is given by
\begin{align*}
{\nabla} l(w) =\left(w^T\nabla_\theta \log\pi_{\theta^{j+1}_0}(a|s)\right)\nabla_{\theta}\log\pi_{\theta^{j+1}_0}(a|s)-u^{j+1}_0.
\end{align*}
Here, $s,a\sim\nu^{\pi_{\theta}}$.

Therefore, By Theorem 1 of \cite{bach2013non}, we know that in order  to reach $\mathbb{E}[\|w^{j+1}_0-(w^{j+1}_0)_{\star}\|^2]\leq \varepsilon'$, we need
\[
\EE[l(w_{\text{out}})-l^{\star}]\leq \frac{2(\xi\sqrt{d}+G\|w^{j+1}_{0, \star}\|)^2}{T}\leq\frac{\mu_F}{2}\varepsilon',
\] 
where $l^{\star}$ is the minimum of $l(w)$, and $\xi$ is defined such that the stochastic gradient $g_{\star}$ at the solution $w^{j+1}_{0,\star}$ satisfies
\[
\EE [g_{\star}{(g_{\star})^T}] \preccurlyeq \xi^2 \nabla^2_{w} l(w).
\]
Similar as Proposition \ref{prop: SGD convergence NPG}, we know that $\xi$ can be chosen by
\[
\xi^2 = \frac{\left(\frac{{2}}{\mu_F}\frac{GR}{(1-\gamma)^2}G^2+\frac{{2}GR}{(1-\gamma)^2}\right)^2}{\mu_F}.
\]
As a result, the number of iterations, $T$, should be
\[
T = \frac{4\left(\frac{\frac{{2}}{\mu_F}\frac{GR}{(1-\gamma)^2}G^2+\frac{{2}GR}{(1-\gamma)^2}}{\sqrt{\mu_F}}\sqrt{d}+\frac{{2}G^2R}{\mu_F(1-\gamma)^2}\right)^2}{\mu_F\varepsilon'} = \cO\left(\frac{1}{(1-\gamma)^4\varepsilon'}\right).
\]
Since each stochastic gradient of $l(w)$ only needs to sample a state-action pair, this is equivalent to sampling $\cO\left(\frac{1}{(1-\gamma)^3\varepsilon'}\right)$ trajectories.

Now, let us turn to $t\geq 1$. $u^{j+1}_t$ is an unbiased estimate of $\nabla J (\theta^{j+1}_t)$, and its variance is bounded as in Lemma \ref{lem: nabla J and v}. Therefore,
\begin{align}
\label{equ: NPG induction}
\begin{split}
\EE[\|w^{j+1}_{t,\star})\|^2]&=\EE[\|F^{-1}(\theta^{j+1}_t)u^{j+1}_t\|^2]\leq \frac{1}{\mu_F^2}\EE[\|u^{j+1}_t\|]\\
&= \frac{1}{\mu_F^2}\EE[\|\nabla  J^H(\theta^{j+1}_t)\|^2]+\frac{1}{\mu_F^2}\EE[\|u^{j+1}_t-\nabla  J^H(\theta^{j+1}_t)\|^2]\\
&\leq \frac{1}{\mu_F^2}\left(\frac{GR}{(1-\gamma)^2}\right)^2+ \frac{1}{\mu_F^2}\left(\frac{C_{\gamma}}{B}\sum_{l=1}^t \mathbb{E}\|\theta^{j+1}_{l}-\theta^{j+1}_{l-1}\|^2+\frac{\sigma^2}{N}\right)\\
&= \frac{1}{\mu_F^2}\left(\frac{GR}{(1-\gamma)^2}\right)^2+ \frac{1}{\mu_F^2}\left(\frac{C_{\gamma}}{B}\sum_{l=1}^t \eta^2\mathbb{E}\|w^{j+1}_{l-1}\|^2+\frac{\sigma^2}{N}\right)\\
&\leq \frac{1}{\mu_F^2}\left(\frac{GR}{(1-\gamma)^2}\right)^2\\
&\,\,\, + \frac{1}{\mu_F^2}\left(\frac{C_{\gamma}}{B}\sum_{l=1}^t 2\eta^2\left(\mathbb{E}\|w^{j+1}_{l-1,\star}\|^2+\mathbb{E}\|w^{j+1}_{l-1,\star}-w^{j+1}_{l-1}\|^2\right)+\frac{\sigma^2}{N}\right).
\end{split}
\end{align}



Now, we are ready to prove the desired results by induction. 

Assume that for all $t'<t$, we have 
\[
\EE[\|w^{j+1}_{t',\star})\|]\leq \frac{4}{\mu_F^2}\left(\frac{GR}{(1-\gamma)^2}\right)^2,
\]
and we have applied 
\[
T = \frac{4\left(\frac{\frac{{2}}{\mu_F}\frac{GR}{(1-\gamma)^2}G^2+\frac{{2}GR}{(1-\gamma)^2}}{\sqrt{\mu_F}}\sqrt{d}+\frac{{2}G^2R}{\mu_F(1-\gamma)^2}\right)^2}{\mu_F\varepsilon'} = \cO\left(\frac{1}{(1-\gamma)^4\varepsilon'}\right)
\]
iterations of SGD as in Procedure \ref{alg: SGD for SRVR-NPG subproblem}. 

Similar to the case of $t=0$, we know that this yields
\[
\mathbb{E}[\|w^{j+1}_{t'}-w^{j+1}_{t',\star}\|^2]\leq \varepsilon'.
\]
Then, by \eqref{equ: NPG induction}, we have
\begin{align*}
\EE[\|w^{j+1}_{t,\star}\|^2]
&\leq \frac{1}{\mu_F^2}\left(\frac{GR}{(1-\gamma)^2}\right)^2\\
&\,\,\, + \frac{1}{\mu_F^2}\left(\frac{C_{\gamma}}{B}\sum_{l=1}^t 2\eta^2\left(\mathbb{E}\|w^{j+1}_{l-1,\star}\|^2+\mathbb{E}\|w^{j+1}_{l-1,\star}-w^{j+1}_{l-1}\|^2\right)+\frac{\sigma^2}{N}\right)\\
&\leq \frac{1}{\mu_F^2}\left(\frac{GR}{(1-\gamma)^2}\right)^2+ \frac{1}{\mu_F^2}\left(\frac{C_{\gamma}m}{B} 2\eta^2\left(\frac{4}{\mu_F^2}\left(\frac{GR}{(1-\gamma)^2}\right)^2 + \varepsilon'\right)+\frac{\sigma^2}{N}\right)\\
&= \frac{1}{\mu_F^2}\left(\frac{GR}{(1-\gamma)^2}\right)^2+ \frac{1}{\mu_F^2}\frac{C_{\gamma}m}{B} 2\eta^2\frac{4}{\mu_F^2}\left(\frac{GR}{(1-\gamma)^2}\right)^2 \\
&\,\,\, + \frac{1}{\mu_F^2}\frac{C_{\gamma}m}{B} 2\eta^2\varepsilon'+\frac{1}{\mu_F^2}\frac{\sigma^2}{N}\\
&\leq \frac{4}{\mu_F^2}\left(\frac{GR}{(1-\gamma)^2}\right)^2.
\end{align*}
As a result, we can apply 
\[
T = \frac{4\left(\frac{\frac{{2}}{\mu_F}\frac{GR}{(1-\gamma)^2}G^2+\frac{{2}GR}{(1-\gamma)^2}}{\sqrt{\mu_F}}\sqrt{d}+\frac{{2}G^2R}{\mu_F(1-\gamma)^2}\right)^2}{\mu_F\varepsilon'} = \cO\left(\frac{1}{(1-\gamma)^4\varepsilon'}\right)
\]
iterations of SGD as in Procedure \ref{alg: SGD for SRVR-NPG subproblem} so that
\[
\mathbb{E}[\|w^{j+1}_{t}-w^{j+1}_{t,\star}\|^2]\leq \varepsilon'.
\]
Since each stochastic gradient of $l(w)$ has a cost of $\frac{1}{1-\gamma}$ (see App. \ref{app: sampling}), this is equivalent to sample 
\[
\frac{4\left(\frac{\frac{{2}}{\mu_F}\frac{GR}{(1-\gamma)^2}G^2+\frac{{2}GR}{(1-\gamma)^2}}{\sqrt{\mu_F}}\sqrt{d}+\frac{{2}G^2R}{\mu_F(1-\gamma)^2}\right)^2}{\mu_F\varepsilon'} = \cO\left(\frac{1}{(1-\gamma)^4\varepsilon'}\right)
\]
trajectories.
\end{proof}


We are now ready to prove Theorem \ref{thm: SRVR-NPG stationary convergence}.
\begin{proof}[Proof of Theorem \ref{thm: SRVR-NPG stationary convergence}]
Line 9 of Algorithm \ref{alg: SRVR-NPG} reads
\begin{align*}
w^{j+1}_{t}\approx \argmin_{w}\{\mathbb{E}_{(s,a)\sim \nu^{j+1}_t}[w^T\nabla_{\theta}\log \pi_{\theta^{j+1}_t}(s,a)]^2-2\langle \eta w, u^{j+1}_t\rangle\}.
\end{align*}
And we want to apply SGD as in Procedure \ref{alg: SGD for SRVR-NPG subproblem} to obtain a $w^{j+1}_{t}$ that satisfies
\begin{align}
\label{equ: subproblem accuracy SRVR-NPG}
\EE[\|w^{j+1}_{t}-F^{-1}(\theta^{j+1}_t)u^{j+1}_t\|^2]\leq \frac{\varepsilon}{3\left(\frac{8G^2\mu_F}{4}+\frac{8G^4}{4}\right)}.
\end{align}
Recall that the parameters $S, m, B$ and $N$ are chosen as
\begin{align*}
S &= \frac{24G^2(J^{\star}- J^H(\theta_0))}{\eta \varepsilon^{0.5}},\\
m & = \frac{1}{\varepsilon^{0.5}},\\
B &= \left(\frac{\eta}{\mu_F}+\frac{\eta}{4G^2}\right)\frac{4C_{\gamma}m}{L_J\varepsilon^{0.25}},\\
N &= 3\left(\frac{8G^2}{\mu_F}+2\right)\frac{\sigma^2}{\varepsilon}.
\end{align*}
Since
\begin{align*}
\varepsilon\leq \min\Big\{ & 3\left(\frac{8G^2}{\mu_F}+2\right)\left(\frac{GR}{(1-\gamma)^2}\right)^2, 3\left(\frac{8G^2}{4}+\frac{8G^4}{4\mu_F}\right)\frac{2}{\mu_F}\left(\frac{GR}{(1-\gamma)^2}\right)^2, \\
& \left(\frac{2}{3\eta L_J}(\mu_F+\frac{\mu_F^2}{4G^2})\right)^4 \Big\},
\end{align*}
the requirements of Proposition \ref{prop: SGD convergence SRVR-NPG} are satisfied:
\begin{align*}
\frac{\sigma^2}{N}& = \frac{\varepsilon}{3\left(\frac{8G^2}{\mu_F}+2\right)}\leq  \left(\frac{GR}{(1-\gamma)^2}\right)^2,\\
\varepsilon' & = \frac{\varepsilon}{3\left(\frac{8G^2\mu_F}{4}+\frac{8G^4}{4}\right)}\leq  \frac{2}{\mu_F^2}\left(\frac{GR}{(1-\gamma)^2}\right)^2\\
\frac{C_\gamma m}{B}2\eta^2& =\frac{2\eta^2\varepsilon^{0.25}}{\left(\frac{\eta}{\mu_F}+\frac{\eta}{4G^2}\right)\frac{4}{L_J}} \leq \frac{1}{3}\mu_F^2. 
\end{align*}
By applying Proposition \ref{prop: SGD convergence SRVR-NPG}, we know that in order to have \eqref{equ: subproblem accuracy SRVR-NPG}, one needs to sample
\[
\frac{4\left(\frac{\frac{\sqrt{2}}{\mu_F}\frac{GR}{(1-\gamma)^2}G^2+\frac{\sqrt{2}GR}{(1-\gamma)^2}}{\sqrt{\mu_F}}\sqrt{d}+\frac{\sqrt{2}G^2R}{\mu_F(1-\gamma)^2}\right)^2}{\mu_F\varepsilon'} = \cO\left(\frac{1}{(1-\gamma)^4\varepsilon}\right)
\]
trajectories. By \eqref{equ: subproblem accuracy SRVR-NPG} we know that 
\begin{align}
\label{equ: point accuracy SRVR-NPG}
\EE[\|\theta^{j+1}_{t+1}-\theta^{j+1}_{t+1, \star}\|^2]\leq \frac{\varepsilon}{3\left(\frac{8G^2\mu_F}{4\eta^2}+\frac{8G^4}{4\eta^2}\right)}
\end{align}
where $\theta^{j+1}_{t+1, \star} = \theta^{j+1}_t+\eta F^{-1}(\theta^{j+1}_t)u^{j+1}_t$. 

On the other hand, we have

\begin{align*}
\begin{split}
 J^H(\theta^{j+1}_{t+1})&\geq  J^H(\theta^{j+1}_t)+\langle \nabla  J^H(\theta^{j+1}_t), \theta^{j+1}_{t+1}-\theta^{j+1}_t\rangle-\frac{L_J}{2}\|\theta^{j+1}_{t+1}-\theta^{j+1}_t\|^2\\
&=  J^H(\theta^{j+1}_t)+\langle \nabla  J^H(\theta^{j+1}_t) - u^{j+1}_t, \theta^{j+1}_{t+1}-\theta^{j+1}_t\rangle \\
&\,\,\, + \langle u^{j+1}_t, \theta^{j+1}_{t+1}-\theta^{j+1}_t\rangle -\frac{L_J}{2}\|\theta^{j+1}_{t+1}-\theta^{j+1}_t\|^2\\
&\geq  J^H(\theta^{j+1}_t)-\frac{\eta}{\mu_F}\|\nabla  J^H(\theta^{j+1}_t) - u^{j+1}_t\|^2-\frac{\mu_F}{4\eta}\|\theta^{j+1}_{t+1}-\theta^{j+1}_t\|^2\\
&\,\,\, +\langle u^{j+1}_t, \theta^{j+1}_{t+1}-\theta^{j+1}_t\rangle -\frac{L_J}{2}\|\theta^{j+1}_{t+1}-\theta^{j+1}_t\|^2
\end{split}
\end{align*}
where we have applied Lemma \ref{lem: smoothness of objective} in the first inequality, and Cauchy-Schwartz in the second one. 

Rearranging gives 
\begin{align*}
 J^H (\theta^{j+1}_{t+1})&\geq   J^H(\theta^{j+1}_t)-\frac{\eta}{\mu_F}\|\nabla  J^H(\theta^{j+1}_t) - u^{j+1}_t\|^2-\frac{\mu_F}{4\eta}\|\theta^{j+1}_{t+1}-\theta^{j+1}_t\|^2\\
&\,\,\, +\frac{1}{2}\langle u^{j+1}_t, \theta^{j+1}_{t+1, \star}-\theta^{j+1}_t\rangle \\
&\,\,\,+\frac{1}{2}\langle u^{j+1}_t, \theta^{j+1}_{t+1}-\theta^{j+1}_{t+1,\star}\rangle + \frac{1}{2}\langle \frac{1}{\eta}F(\theta^{j+1}_t)(\theta^{j+1}_{t+1, \star}-\theta^{j+1}_t), \theta^{j+1}_{t+1}-\theta^{j+1}_t\rangle \\
&\,\,\, - \frac{L_J}{2}\|\theta^{j+1}_{t+1}-\theta^{j+1}_t\|^2\\
&=  J^H(\theta^{j+1}_t)-\frac{\eta}{\mu_F}\|\nabla  J^H(\theta^{j+1}_t) - u^{j+1}_t\|^2\\
&\,\,\, -\frac{\mu_F}{4\eta}\|\theta^{j+1}_{t+1}-\theta^{j+1}_t\|^2+\frac{1}{2}\langle u^{j+1}_t, \eta F^{-1}(\theta^{j+1}_t)u^{j+1}_t\rangle \\
&\,\,\,+\frac{1}{2}\langle u^{j+1}_t, \theta^{j+1}_{t+1}-\theta^{j+1}_{t+1,\star}\rangle +\frac{1}{2}\langle \frac{1}{\eta}F(\theta^{j+1}_t)(\theta^{j+1}_{t+1, \star}-\theta^{j+1}_t), \theta^{j+1}_{t+1}-\theta^{j+1}_t\rangle\\
&\,\,\, -\frac{L_J}{2}\|\theta^{j+1}_{t+1}-\theta^{j+1}_t\|^2,	
\end{align*}


Applying $F^{-1}(\theta) \succeq \frac{1}{G^2} I$ on the first inner product, and Cauchy-Schwartz on the second inner product term leads to
\begin{align*}
\begin{split}
 J^H(\theta^{j+1}_{t+1})
&\geq  J^H(\theta^{j+1}_t)-\frac{\eta}{\mu_F}\|\nabla  J^H(\theta^{j+1}_t) - u^{j+1}_t\|^2\\
&\,\,\, -\frac{\mu_F}{4\eta}\|\theta^{j+1}_{t+1}-\theta^{j+1}_t\|^2+\frac{\eta}{2G^2}\|u^{j+1}_t\|^2  \\
&\,\,\,-\frac{\eta}{4G^2} \|u^{j+1}_t\|^2 - \frac{G^2}{4\eta}\|\theta^{j+1}_{t+1}-\theta^{j+1}_{t+1,\star}\|^2\\
&\,\,\, +\frac{1}{2}\langle \frac{1}{\eta}F(\theta^{j+1}_t)(\theta^{j+1}_{t+1, \star}-\theta^{j+1}_t), \theta^{j+1}_{t+1}-\theta^{j+1}_t\rangle-\frac{L_J}{2}\|\theta^{j+1}_{t+1}-\theta^{j+1}_t\|^2\\
&=  J^H(\theta^{j+1}_t)-\frac{\eta}{\mu_F}\|\nabla  J^H(\theta^{j+1}_t) - u^{j+1}_t\|^2\\
&\,\,\, -\left(\frac{\mu_F}{4\eta}+\frac{L_J}{2}\right)\|\theta^{j+1}_{t+1}-\theta^{j+1}_t\|^2+\frac{\eta}{4G^2}\|u^{j+1}_t\|^2  \\
&\,\,\,+\frac{1}{2}\langle \frac{1}{\eta}F(\theta^{j+1}_t)(\theta^{j+1}_{t+1, \star}-\theta^{j+1}_t), \theta^{j+1}_{t+1}-\theta^{j+1}_t\rangle- \frac{G^2}{4\eta}\|\theta^{j+1}_{t+1}-\theta^{j+1}_{t+1,\star}\|^2\\
&=  J^H(\theta^{j+1}_t)\\
&\,\,\, -\frac{\eta}{\mu_F}\|\nabla  J^H(\theta^{j+1}_t) - u^{j+1}_t\|^2-\left(\frac{\mu_F}{4\eta}+\frac{L_J}{2}\right)\|\theta^{j+1}_{t+1}-\theta^{j+1}_t\|^2+\frac{\eta}{4G^2}\|u^{j+1}_t\|^2  \\
&\,\,\,+\frac{1}{2}\langle \frac{1}{\eta}F(\theta^{j+1}_t)(\theta^{j+1}_{t+1}-\theta^{j+1}_t), \theta^{j+1}_{t+1}-\theta^{j+1}_t\rangle \\
&\,\,\, +\frac{1}{2}\langle \frac{1}{\eta}F(\theta^{j+1}_t)(\theta^{j+1}_{t+1, \star}-\theta^{j+1}_{t+1}), \theta^{j+1}_{t+1}-\theta^{j+1}_t\rangle\\
&\,\,\,- \frac{G^2}{4\eta}\|\theta^{j+1}_{t+1}-\theta^{j+1}_{t+1,\star}\|^2.
\end{split}
\end{align*}
Applying $\|\nabla  J^H(\theta^{j+1}_t)\|^2\leq 2\|\nabla  J^H(\theta^{j+1}_t)-u^{j+1}_t\|^2+2\|u^{j+1}_t\|^2$ and $F(\theta^{j+1}_t) \succeq \mu_F I_d$ yields
\begin{align}
\label{equ: V ascent SRVR-NPG}
\begin{split}
 J^H(\theta^{j+1}_{t+1}) &\geq  J^H(\theta^{j+1}_t)-\left(\frac{\eta}{\mu_F}+\frac{\eta}{4G^2}\right)\|\nabla  J^H(\theta^{j+1}_t) - u^{j+1}_t\|^2\\
&\,\,\, -\left(\frac{\mu_F}{4\eta}+\frac{L_J}{2}\right)\|\theta^{j+1}_{t+1}-\theta^{j+1}_t\|^2+\frac{\eta}{8G^2}\|\nabla  J^H(\theta^{j+1}_t)\|^2  \\
&\,\,\,+\frac{\mu_F}{2\eta}\| \theta^{j+1}_{t+1}-\theta^{j+1}_t\|^2 +\frac{1}{2}\langle \frac{1}{\eta}F(\theta^{j+1}_t)(\theta^{j+1}_{t+1, \star}-\theta^{j+1}_{t+1}), \theta^{j+1}_{t+1}-\theta^{j+1}_t\rangle\\
&\,\,\,- \frac{G^2}{4\eta}\|\theta^{j+1}_{t+1}-\theta^{j+1}_{t+1,\star}\|^2\\
&\geq  J^H(\theta^{j+1}_t)-\left(\frac{\eta}{\mu_F}+\frac{\eta}{4G^2}\right)\|\nabla  J^H(\theta^{j+1}_t) - u^{j+1}_t\|^2\\
&\,\,\, -\left(\frac{\mu_F}{4\eta}+\frac{L_J}{2}\right)\|\theta^{j+1}_{t+1}-\theta^{j+1}_t\|^2+\frac{\eta}{8G^2}\|\nabla  J^H(\theta^{j+1}_t)\|^2  \\
&\,\,\,+\frac{\mu_F}{2\eta}\| \theta^{j+1}_{t+1}-\theta^{j+1}_t\|^2 -\frac{\mu_F^2}{2\mu_F\eta}\|\theta^{j+1}_{t+1, \star}-\theta^{j+1}_{t+1}\|^2- \frac{\mu_F}{8\eta}\|\theta^{j+1}_{t+1}-\theta^{j+1}_t\|^2\\
&\,\,\,- \frac{G^2}{4\eta}\|\theta^{j+1}_{t+1}-\theta^{j+1}_{t+1,\star}\|^2\\
&=  J^H(\theta^{j+1}_t)-\left(\frac{\eta}{\mu_F}+\frac{\eta}{4G^2}\right)\|\nabla  J^H(\theta^{j+1}_t) - u^{j+1}_t\|^2\\
&\,\,\, +\left(\frac{\mu_F}{8\eta}-\frac{L_J}{2}\right)\|\theta^{j+1}_{t+1}-\theta^{j+1}_t\|^2+\frac{\eta}{8G^2}\|\nabla  J^H(\theta^{j+1}_t)\|^2  \\
&\,\,\,-(\frac{\mu_F}{4\eta}+\frac{G^2}{4\eta})\|\theta^{j+1}_{t+1, \star}-\theta^{j+1}_{t+1}\|^2.
\end{split}
\end{align}

From Lemma \ref{lem: nabla J and v}, we further know that 

\begin{align*}
\EE [ J^H(\theta^{j+1}_{t+1})]&\geq \EE[ J^H(\theta^{j+1}_t)]-\left(\frac{\eta}{\mu_F}+\frac{\eta}{4G^2}\right)\left(\frac{C_{\gamma}}{B}\sum_{t=0}^{m-1} \mathbb{E}\|\theta^{j+1}_{t+1}-\theta^{j+1}_{t}\|^2+\frac{\sigma^2}{N}\right)\\
&\,\,\, +\left(\frac{\mu_F}{8\eta}-\frac{L_J}{2}\right)\EE\|\theta^{j+1}_{t+1}-\theta^{j+1}_t\|^2  \\
&\,\,\,+\frac{\eta}{8G^2}\EE\|\nabla  J^H(\theta^{j+1}_t)\|^2-(\frac{\mu_F}{4\eta}+\frac{G^2}{4\eta})\EE\|\theta^{j+1}_{t+1, \star}-\theta^{j+1}_{t+1}\|^2.
\end{align*}
Telescoping for $s=0,1,...,S-1$ and $t=0,1,...,m-1$ and dividing by $Sm$ gives 
\begin{align}
\label{equ: SRVR-NPG final}
\begin{split}
&\frac{8G^2}{\eta}\left(\frac{\mu_F}{8\eta} - \frac{L_J}{2}-\left(\frac{\eta}{\mu_F}+\frac{\eta}{4G^2}\right)\frac{C_{\gamma}m}{B}\right)\frac{1}{Sm}\sum_{s=0}^{S-1}\sum_{t=0}^{m-1}\EE\|\theta^{j+1}_{t+1}-\theta^{j+1}_t\|^2\\
&\,\,\, +\frac{1}{Sm} \sum_{s=0}^{S-1}\sum_{t=0}^{m-1} \EE\|\nabla  J^H(\theta^{j+1}_t)\|^2\\ 
&\leq \frac{8G^2}{\eta}\frac{J^{\star}- J^H(\theta_0)}{Sm} + \left(\frac{8G^2}{\mu_F}+2\right)\frac{\sigma^2}{N}+\left(\frac{8G^2\mu_F}{4\eta^2}+\frac{8G^4}{4\eta^2}\right)\frac{1}{Sm}\sum_{s=0}^{S-1}\sum_{t=0}^{m-1}\EE\|\theta^{j+1}_{t+1, \star}-\theta^{j+1}_{t+1}\|^2.
\end{split}
\end{align}
Let us first show that the first term on the left hand side of \eqref{equ: SRVR-NPG final} is non-negative. In fact, from $\eta = \frac{\mu_F}{8L_J}$ and $B = \left(\frac{\eta}{\mu_F}+\frac{\eta}{4G^2}\right)\frac{4C_{\gamma}m}{L_J\varepsilon^{0.25}}$ we have
\[
\frac{8G^2}{\eta}\left(\frac{\mu_F}{8\eta} - \frac{L_J}{2}-\left(\frac{\eta}{\mu_F}+\frac{\eta}{4G^2}\right)\frac{C_{\gamma}m}{B}\right)\geq\frac{16G^2L_J^2}{\mu_F}>0.
\]
Therefore, in order to have 
\[
\frac{1}{Sm} \sum_{s=0}^{S-1}\sum_{t=0}^{m-1} \EE\|\nabla  J^H(\theta^{j+1}_t)\|^2\leq \varepsilon,
\]
we can set all the three terms on the right hand side of \eqref{equ: SRVR-NPG final} to be $\frac{\varepsilon}{3}$, which gives
\begin{align*}
Sm & = \frac{24G^2( J^{H, \star}- J^H(\theta_0))}{\eta \varepsilon},\\
N &= 3\left(\frac{8G^2}{\mu_F}+2\right)\frac{\sigma^2}{\varepsilon}\\
\EE[\|\theta^{j+1}_{t+1}-\theta^{j+1}_{t+1, \star}\|^2]&\leq \frac{\varepsilon}{3\left(\frac{8G^2\mu_F}{4\eta^2}+\frac{8G^4}{4\eta^2}\right)},
\end{align*}
where the last requirement is satisfied according to \eqref{equ: point accuracy SRVR-NPG}. 

For the parameters $S, m, B$, and $N$, we have 
\begin{align*}
S &= \frac{24G^2( J^{H, \star}- J^H(\theta_0))}{\eta \varepsilon^{0.5}} = \cO\left(\frac{1}{(1-\gamma)^{2}\varepsilon^{0.5}}\right),\\
m & = \frac{1}{\varepsilon^{0.5}},\\
B &= \left(\frac{\eta}{\mu_F}+\frac{\eta}{4G^2}\right)\frac{4C_{\gamma}m}{L_J\varepsilon^{0.25}} = \cO\left(\frac{W}{(1-\gamma) \varepsilon^{0.75}}\right),\\
N &= 3\left(\frac{8G^2}{\mu_F}+2\right)\frac{\sigma^2}{\varepsilon}=\cO\left(\frac{\sigma^2}{\varepsilon}\right),
\end{align*}
where we have applied the definition of $C_{\gamma}$ in Lemma \ref{lem: nabla J and v} in the third equality.

Therefore in total, the number of trajectories required by SRVR-NPG to reach $\varepsilon-$stationarity is 
\begin{align*}
& S\left(N+mB+(m+1)\frac{4\left((\frac{2 G^2R}{\mu_F(1-\gamma)^2}+\frac{2}{1-\gamma})\sqrt{d}+\frac{2G^2R}{\mu_F(1-\gamma)^2}\right)^2}{\mu_F\frac{\varepsilon}{3\left(\frac{8G^2\mu_F}{4}+\frac{8G^4}{4}\right)}}\right) \\
& = \cO\left(\frac{\sigma^2}{(1-\gamma)^{2}\varepsilon^{1.5}}+\frac{W}{(1-\gamma)^{3}\varepsilon^{1.75}}+\frac{1}{(1-\gamma)^6\varepsilon^{2}}\right).
\end{align*}

\end{proof}

\section{Proof of Proposition \ref{prop: global convergence}}
\label{app: global}
In this section, we proceed to prove Proposition \ref{prop: global convergence}, which establishes a general global convergence result on policy gradient methods of the form $\theta^{k+1} = \theta^k +\eta w^k$.  



\begin{proof}
First, by the $M-$smoothness of score function (see Assumption \ref{assump: conditions on score function}), we know that
\begin{align*}
&\mathbb{E}_{s\sim d^{\pi^{\star}}_{\rho}} \left[\text{KL}\left(\pi^{\star}(\cdot\given s)|| \pi_{\theta^{k}}(\cdot\given s)\right)-\text{KL}\left(\pi^{\star}(\cdot\given s)|| \pi_{\theta^{k+1}}(\cdot \given s)\right)\right]\\
& = \mathbb{E}_{s\sim d^{\pi^{\star}}_{\rho}}\mathbb{E}_{a\sim \pi^{\star} (\cdot \given s)}\left[\log\frac{\pi_{\theta^{k+1}}(a\given s)}{\pi_{\theta^{k}}(a\given s)}\right]\\
&\geq \mathbb{E}_{s\sim d^{\pi^{\star}}_{\rho}}\mathbb{E}_{a\sim \pi^{\star} (\cdot \given s)} [\nabla_{\theta}\log \pi_{\theta^k}(a\given s)\cdot(\theta^{k+1}-\theta^k)]-\frac{M}{2}\|\theta^{k+1}-\theta^k\|^2\\
&= \eta \mathbb{E}_{s\sim d^{\pi^{\star}}_{\rho}}\mathbb{E}_{a\sim \pi^{\star} (\cdot \given s)} [\nabla_{\theta}\log \pi_{\theta^k}(a\given s)\cdot w^k]-\frac{M\eta^2}{2}\|w^k\|^2\\
&= \eta \mathbb{E}_{s\sim d^{\pi^{\star}}_{\rho}}\mathbb{E}_{a\sim \pi^{\star} (\cdot \given s)} [\nabla_{\theta}\log \pi_{\theta^k}(a\given s)\cdot w^k_{\star}] \\
&\,\,\, + \eta \mathbb{E}_{s\sim d^{\pi^{\star}}_{\rho}}\mathbb{E}_{a\sim \pi^{\star} (\cdot \given s)} [\nabla_{\theta}\log \pi_{\theta^k}(a\given s)\cdot (w^k - w^k_{\star})] -\frac{M\eta^2}{2}\|w^k\|^2\\
\end{align*}
where we have applied $\text{KL}\left( p||q \right) = \EE_{x\sim p}[-\log\frac{q(x)}{p(x)}]$ in the first step.

On the other hand, by the performance difference lemma \cite{kakade2002approximately} we know that
\begin{align*}
\mathbb{E}_{s\sim d^{\pi^{\star}}_{\rho}}\mathbb{E}_{a\sim \pi^{\star} (\cdot \given s)} [A^{\pi_{\theta^k}}(s,a)] = (1-\gamma) \left(J^{\star}-J(\theta^k)\right).
\end{align*}
Therefore, 
\begin{align*}
&\mathbb{E}_{s\sim d^{\pi^{\star}}_{\rho}} \left[\text{KL}\left(\pi^{\star}(\cdot\given s)|| \pi_{\theta^k}(\cdot\given s)\right)-\text{KL}\left(\pi^{\star}(\cdot\given s)|| \pi_{\theta^{k+1}}(\cdot \given s)\right)\right]\\
&\geq \eta \mathbb{E}_{s\sim d^{\pi^{\star}}_{\rho}}\mathbb{E}_{a\sim \pi^{\star} (\cdot \given s)} [\nabla_{\theta}\log \pi_{\theta^k}(a\given s)\cdot w^k_{\star}] \\
&\,\,\, + \eta \mathbb{E}_{s\sim d^{\pi^{\star}}_{\rho}}\mathbb{E}_{a\sim \pi^{\star} (\cdot \given s)} [\nabla_{\theta}\log \pi_{\theta^k}(a\given s)\cdot (w^k - w^k_{\star})] -\frac{M\eta^2}{2}\|w^k\|^2\\
&= \eta \left(J^{\star}-J(\theta^k)\right)+\eta \mathbb{E}_{s\sim d^{\pi^{\star}}_{\rho}}\mathbb{E}_{a\sim \pi^{\star} (\cdot \given s)} [\nabla_{\theta}\log \pi_{\theta^k}(a\given s)\cdot w^k_{\star} - \frac{1}{1-\gamma}A^{\pi_{\theta^k}}(s,a)]\\
&\,\,\, +\eta \mathbb{E}_{s\sim d^{\pi^{\star}}_{\rho}}\mathbb{E}_{a\sim \pi^{\star} (\cdot \given s)} [\nabla_{\theta}\log \pi_{\theta^k}(a\given s)\cdot (w^k - w^k_{\star})] -\frac{M\eta^2}{2}\|w^k\|^2\\
&= \eta \left(J^{\star}-J(\theta^k)\right)+\eta\frac{1}{1-\gamma} \mathbb{E}_{s\sim d^{\pi^{\star}}_{\rho}}\mathbb{E}_{a\sim \pi^{\star} (\cdot \given s)} [\nabla_{\theta}\log \pi_{\theta^k}(a\given s)\cdot (1-\gamma)w^k_{\star} - A^{\pi_{\theta^k}}(s,a)]\\
&\,\,\, +\eta  \mathbb{E}_{s\sim d^{\pi^{\star}}_{\rho}}\mathbb{E}_{a\sim \pi^{\star} (\cdot \given s)} [\nabla_{\theta}\log \pi_{\theta^k}(a\given s)\cdot (w^k - w^k_{\star})] -\frac{M\eta^2}{2}\|w^k\|^2.
\end{align*}
Now, let us apply Jensen's inequality and Assumption \ref{assump: conditions on score function} to obtain
\begin{align*}
\begin{split}
&\mathbb{E}_{s\sim d^{\pi^{\star}}_{\rho}} \left[\text{KL}\left(\pi^{\star}(\cdot\given s)|| \pi_{\theta^k}(\cdot\given s)\right)-\text{KL}\left(\pi^{\star}(\cdot\given s)|| \pi_{\theta^{k+1}}(\cdot \given s)\right)\right]\\
&\geq \eta \left(J^{\star}-J(\theta^k)\right) \\
&\,\,\, - \eta\frac{1}{1-\gamma} \sqrt{\mathbb{E}_{s\sim d^{\pi^{\star}}_{\rho}}\mathbb{E}_{a\sim \pi^{\star} (\cdot \given s)} \left[\left(\nabla_{\theta}\log \pi_{\theta^k}(a\given s)\cdot (1-\gamma)w^k_{\star} - A^{\pi_{\theta^k}}(s,a)\right)^2\right]}\\
&\,\,\, - \eta G  \|w^k - w^k_{\star}\| -\frac{M\eta^2}{2}\|w^k\|^2
\end{split}
\end{align*}
Combining this with Assumption \ref{assump: compatible error} yields
\begin{align}
\label{equ: global 1}
\begin{split}
&\mathbb{E}_{s\sim d^{\pi^{\star}}_{\rho}} \left[\text{KL}\left(\pi^{\star}(\cdot\given s)|| \pi_{\theta^k}(\cdot\given s)\right)-\text{KL}\left(\pi^{\star}(\cdot\given s)|| \pi_{\theta^{k+1}}(\cdot \given s)\right)\right]\\
&\geq \eta \left(J^{\star}-J(\theta^k)\right) - \eta \sqrt{\frac{1}{(1-\gamma)^2}\varepsilon_{\text{bias}}}- \eta G  \|w^k - w^k_{\star}\| -\frac{M\eta^2}{2}\|w^k\|^2.	
\end{split}
\end{align}
Finally, let us telescope the above inequality from $k=0$ to $K-1$, and divide by $K$, which gives 
\begin{align}
\label{equ: equ of global convergence}
\begin{split}
J(\pi^{\star})-\frac{1}{K}\sum_{k=0}^{K-1}J(\theta^k)
&\leq \frac{\sqrt{\varepsilon_{\text{bias}}}}{1-\gamma} + \frac{1}{\eta K} \mathbb{E}_{s\sim d^{\pi^{\star}}_\rho} \left[\text{KL}\left(\pi^{\star}(\cdot\given s)|| \pi_{\theta^0}(\cdot\given s)\right)\right]\\
&\,\,\, + \frac{G}{K}\sum_{k=0}^{K-1} \|w^k-w^k_{\star}\|+\frac{M\eta}{2K}\sum_{k=0}^{K-1}\|w^k\|^2.
\end{split}
\end{align}
\end{proof}
On the right hand side of \eqref{equ: equ of global convergence}, the first term reflects the function approximation error due to the possibly imperfect policy parametrization. The second term vanishes as $K\rightarrow \infty$. 

By looking at the third and fourth term, we know that for an update of the form $\theta^{k+1}=\theta^k+\eta w^k$, its global convergence rate depends crucially on i) the difference between its update directions $w^k$ and the exact NPG update direction $w^k_{\star}$, and ii) its stationary convergence rate. 

In the rest of this paper, we shall see that for stochastic PG, NPG, SRVR-PG, and SRVR-NPG, both the third and fourth terms of \eqref{equ: equ of global convergence} go to $0$ as $K\rightarrow \infty$, whose speed lead to different global convergence rates for different algorithms. In order to achieve this, we will apply some intermediate results in the previous proof of stationary convergence. 

\section{Proof of Theorem \ref{thm: PG global convergence}}
\label{app: PG global}

Let us take $w^k$ as the update direction of PG and apply Proposition \ref{prop: global convergence}. To this end, we need to upper bound 
$\frac{1}{K}\sum_{k=0}^{K-1} \|w^k-w^k_{\star}\|$, $\frac{1}{K}\sum_{k=0}^{K-1}\|w^k\|^2$, and $\frac{1}{K} \mathbb{E}_{s\sim d^{\pi^{\star}}_\rho} \left[\text{KL}\left(\pi^{\star}(\cdot\given s)|| \pi_{\theta^0}(\cdot\given s)\right)\right]$, where $w^k_{\star}=F^{-1}(\theta^k)\nabla  J(\theta^k)$ is the exact NPG update direction at $\theta^k$.
\begin{itemize}
\item Bounding $\frac{1}{K}\sum_{k=0}^{K-1} \|w^k-w^k_{\star}\|$.

We know from Jensen's inequality and $\left(\mathbb{E}[\|w^{j+1}_{t}-w^{j+1}_{t, \star}\|]\right)^2\leq \mathbb{E}[\|w^{j+1}_{t}-w^{j+1}_{t, \star}\|^2]$ that
\begin{align}
\label{equ: PG 1111}
\begin{split}
&\left(\frac{1}{K}\sum_{k=0}^{K-1}\EE[\|w^k - w^k_{\star}\|]\right)^2 \\
&\leq \frac{1}{K}\sum_{k=0}^{K-1}\left(\EE[\|w^k - w^k_{\star}\|]\right)^2\\
&\leq \frac{1}{K}\sum_{k=0}^{K-1}\EE[\|w^k - w^k_{\star}\|^2]\\
&\leq \frac{2}{K}\sum_{k=0}^{K-1}\EE[\|w^k - \nabla  J(\theta^k)\|^2] + \frac{2}{K}\sum_{k=0}^{K-1}\EE[\|\nabla  J(\theta^k) - F^{-1}(\theta^k)\nabla  J(\theta^k)\|^2]
\end{split}
\end{align}
Since 
\[
w^k = \frac{1}{N}\sum_{i=1}^N g(\tau^H_i\given \theta^k), 
\]
we have from Lemma \ref{lem: smoothness of objective} and Assumption \ref{assump: variance} that
\begin{align}
\label{equ: PG 2222}
\begin{split}
&\frac{1}{K}\sum_{k=0}^{K-1}\EE[\|w^k-\nabla  J(\theta^k)\|^2]\\
&\leq \frac{2}{K}\sum_{k=0}^{K-1}\EE[\|w^k-\nabla  J^H(\theta^k)\|^2]+ \frac{2}{K}\sum_{k=0}^{K-1}\EE[\|\nabla  J^H(\theta^k)-\nabla  J(\theta^k)\|^2]\\
&\leq 2\frac{\sigma^2}{N} + 2 G^2R^2 \left(\frac{H+1}{1-\gamma}+\frac{\gamma}{(1-\gamma)^2}\right)^2\gamma^{2H}.
\end{split}
\end{align}
Furthermore, Assumption \ref{assump: strong convexity} tells us that
\begin{align}
\label{equ: PG 3333}
\begin{split}
&\EE[\|\nabla  J(\theta^k) - F^{-1}(\theta^k)\nabla  J(\theta^k)\|^2]\\
&\leq \left(1+\frac{1}{\mu_F}\right)^2\EE[\|\nabla  J(\theta^k)\|^2]\\
&\leq \left(1+\frac{1}{\mu_F}\right)^2\left(2\EE[\|\nabla  J^H(\theta^k)\|^2]+2 G^2R^2 \left(\frac{H+1}{1-\gamma}+\frac{\gamma}{(1-\gamma)^2}\right)^2\gamma^{2H}\right).
\end{split}
\end{align}
Combining \eqref{equ: PG 2222} and \eqref{equ: PG 3333} with \eqref{equ: PG 1111} gives 
\begin{align}
\label{equ: PG 4444}
\begin{split}
&\frac{1}{K}\sum_{k=0}^{K-1}\EE[\|w^k - w^k_{\star}\|]\\
&\leq \Bigg(2\frac{\sigma^2}{N} + 2 G^2R^2 \left(\frac{H+1}{1-\gamma}+\frac{\gamma}{(1-\gamma)^2}\right)^2\gamma^{2H}\\
& \,\,\, + 2\left(1+\frac{1}{\mu_F}\right)^2\\
&\,\,\,\qquad \cdot \frac{1}{K}\sum_{k=0}^{K-1}\left(2\EE[\|\nabla  J^H(\theta^k)\|^2]+2G^2R^2 \left(\frac{H+1}{1-\gamma}+\frac{\gamma}{(1-\gamma)^2}\right)^2\gamma^{2H}\right)\Bigg)^{0.5}
\end{split}
\end{align}
And recall from \eqref{equ: PG stationary final} that
\[
\frac{1}{K}\sum_{k=0}^{K-1}\EE[\|\nabla  J^H(\theta^k)\|^2]\leq \frac{\frac{ J^{H, \star}- J^H (\theta_0)}{K}+(\frac{\eta}{2}+L_J\eta^2)\frac{\sigma^2}{N}}{\frac{\eta}{2}-L_J\eta^2}.
\]
Let us take $\eta = \frac{1}{4L_J}$. In addition, let $H$, $N$, and $K$ satisfy
\begin{align}
\label{equ: PG 1}
\begin{split}
\frac{1}{3}(\frac{\varepsilon}{3G})^2 &\geq \left(2+4(1+\frac{1}{\mu_F})^2\right)G^2R^2 \left(\frac{H+1}{1-\gamma}+\frac{\gamma}{(1-\gamma)^2}\right)^2\gamma^{2H}\\
N&\geq \frac{\left(2+12(1+\frac{1}{\mu_F})^2\right)\sigma^2}{\frac{1}{3}\left(\frac{\varepsilon}{3G}\right)^2},\\
K& \geq \frac{\left(1+\frac{1}{\mu_F}\right)^2 64L_J ( J^{H, \star}- J^H(\theta_0))}{\frac{1}{3}\left(\frac{\varepsilon}{3G}\right)^2}.
\end{split}
\end{align}
Then, we have 
\begin{align}
\label{equ: PG final 1}
\frac{G}{K}\sum_{k=0}^{K-1}\EE[\|w^k - w^k_{\star}\|]\leq \frac{\varepsilon}{3}.
\end{align}
\item Bounding $\frac{1}{K}\sum_{k=0}^{K-1}\|w^k\|^2$. 
 
We have from \eqref{equ: PG 2222} and \eqref{equ: PG stationary final} that
\begin{align*}
&\frac{1}{K}\sum_{k=0}^{K-1}\EE\|w^k\|^2\\
&\leq \frac{\sigma^2}{N}+\frac{1}{K}\sum_{k=0}^{K-1} \mathbb{E}[\|\nabla  J^H(\theta^{k})\|^2]\\
& \leq \frac{\sigma^2}{N} + \frac{\frac{ J^{H, \star}- J^H(\theta_0)}{K}+(\frac{\eta}{2}+L_J\eta^2)\frac{\sigma^2}{N}}{\frac{\eta}{2}-L_J\eta^2}.
\end{align*}
Taking $\eta = \frac{1}{4L_j}$ and 
\begin{align}
\label{equ: PG 2}
\begin{split}
N &\geq \frac{12 M\eta\sigma^2}{\varepsilon},\\
K &\geq  \frac{48L_JM\eta( J^{H,\star}- J^H(\theta_0))}{\varepsilon},
\end{split}
\end{align}
we arrive at
\begin{align}
\label{equ: PG final 2}
\frac{M\eta}{2K}\sum_{k=0}^{K-1}\EE[\|w^k\|^2]\leq \frac{\varepsilon}{3}.
\end{align}

\item Bounding $\frac{1}{K} \mathbb{E}_{s\sim d^{\pi^{\star}}_\rho} \left[\text{KL}\left(\pi^{\star}(\cdot\given s)|| \pi_{\theta^0}(\cdot\given s)\right)\right]$.
 
By taking 
\begin{align}
\label{equ: PG 3}
K \geq \frac{3\mathbb{E}_{s\sim d^{\pi^{\star}}_\rho} \left[\text{KL}\left(\pi^{\star}(\cdot\given s)|| \pi_{\theta^0}(\cdot\given s)\right)\right]}{\eta \varepsilon}
\end{align}
we have
\begin{align}
\label{equ: PG final 3}	
\frac{1}{\eta K} \mathbb{E}_{s\sim d^{\pi^{\star}}_\rho} \left[\text{KL}\left(\pi^{\star}(\cdot\given s)|| \pi_{\theta^0}(\cdot\given s)\right)\right] \leq \frac{\varepsilon}{3},
\end{align}
\end{itemize}

In summary, we require $N$ and $K$ to satisfy \eqref{equ: PG 1}, \eqref{equ: PG 2}, and \eqref{equ: PG 3}, which leads to 
\[N = \cO\left(\frac{\sigma^2}{\varepsilon^2}\right), \quad\quad K = \cO\left(\frac{1}{(1-\gamma)^2\varepsilon^2}\right), \qquad H = \cO\left(\log((1-\gamma)^{-1}\varepsilon^{-1})\right).\]
By combining \eqref{equ: PG final 1}, \eqref{equ: PG final 2}, \eqref{equ: PG final 3} and \eqref{equ: equ of global convergence}, we can conclude that
\begin{align*}
\begin{split}
J(\pi^{\star})-\frac{1}{K}\sum_{k=0}^{K-1}J(\theta^k)
&\leq \frac{\sqrt{\varepsilon_{\text{bias}}}}{1-\gamma} + \varepsilon.
\end{split}
\end{align*}
In total, stochastic PG requires to sample $KN=\cO\left(\frac{\sigma^2}{(1-\gamma)^2\varepsilon^4}\right)$ trajectories. 

\section{Proof of Theorem \ref{thm: NPG global convergence}}
\label{app: NPG global}

Let us take $w^k$ as the update direction of NPG and apply Proposition \ref{prop: global convergence}. To this end, we need to upper bound 
$\frac{1}{K}\sum_{k=0}^{K-1} \|w^k-w^k_{\star}\|$, $\frac{1}{K}\sum_{k=0}^{K-1}\|w^k\|^2$, and $\frac{1}{K} \mathbb{E}_{s\sim d^{\pi^{\star}}_\rho} \left[\text{KL}\left(\pi^{\star}(\cdot\given s)|| \pi_{\theta^0}(\cdot\given s)\right)\right]$, where $w^{k}_{\star}=F^{-1}(\theta^k)\nabla  J(\theta^k)$ is the exact NPG update direction at $\theta^k$. 

Let us take $\eta = \frac{\mu_F^2}{4G^2L_J}$ and apply SGD as in Procedure \ref{alg: SGD for NPG subproblem} to obtain a $w^k$ that satisfies
\begin{align}
\label{equ: NPG 111}
\EE[\|w^k-w^k_{\star}\|^2]\leq \min\Big\{\frac{\varepsilon}{12 M\eta}, \left(\frac{1}{G}\frac{\varepsilon}{3}\right)^2, \frac{\varepsilon}{12M\eta^3}\frac{\mu_F^2}{2}\frac{\frac{\eta}{2G^2}-\frac{L_J\eta^2}{\mu_F^2}}{\frac{G^2}{2\eta}+L_J}\Big\}.
\end{align}
From Proposition \ref{prop: SGD convergence NPG}, we know that this requires sampling $\cO\left(\frac{1}{(1-\gamma)^4\varepsilon^2}\right)$ trajectories at each iteration.

\begin{itemize}

\item Bounding $\frac{1}{K}\sum_{k=0}^{K-1} \|w^k-w^k_{\star}\|$.

Recall that the update direction $w^k\approx w^k_{\star}=F^{-1}(\theta^k)\nabla  J(\theta^k)$ is obtained by solving the subproblem
\begin{align*}
w^k&\approx \argmin_{w\in\Rd} 	L_{\nu^{\pi_{\theta}}}(w; \theta^k)\\
& = \argmin_{w\in\Rd}  \EE_{(s,a)\sim\nu^{\pi_{\theta^k}}}\big[A^{\pi_{\theta^k}}(s,a)-(1-\gamma)w^\top\nabla_{\theta}\log\pi_{\theta^k}(a\given s)\big]^2.
\end{align*} 
By \eqref{equ: NPG 111} and Jensen's inequality, we can write
\begin{align}
\label{equ: NPG final 1}
\left(\frac{1}{K}\sum_{k=0}^{K-1}\EE[\|w^k - w^k_{\star}\|]\right)^2&\leq \frac{1}{K}\sum_{k=0}^{K-1}\EE[\|w^k - w^k_{\star}\|^2]\leq \left(\frac{1}{G}\frac{\varepsilon}{3}\right)^2.
\end{align}

On the other hand, by replacing \eqref{equ: NPG subproblem accuracy} with \eqref{equ: NPG 111}, the stationary convergence of NPG stated in \eqref{equ: NPG stationary final} becomes 
\begin{align*}
\frac{1}{K}\sum_{k=0}^{K-1} \mathbb{E}[\|\nabla  J(\theta^{k})\|^2]
&\leq \frac{\frac{ J^{\star}- J(\theta_0)}{K} + \frac{\varepsilon}{12M\eta}\frac{\mu_F^2}{2}\left(\frac{\eta}{2G^2}-\frac{L_J\eta^2}{\mu_F^2}\right)}{\frac{\eta}{2G^2}-\frac{L_J\eta^2}{\mu_F^2}}
\end{align*}

\item Bounding $\frac{1}{K}\sum_{k=0}^{K-1}\|w^k\|^2$.

Taking $\eta = \frac{\mu_F^2}{4G^2L_J}$ and 
\begin{align}
\label{equ: NPG 1}
K \geq \frac{24( J^{\star}- J(\theta_0))M\eta}{\mu_F^2 \left(\frac{\eta}{2G^2}-\frac{L_J\eta^2}{\mu_F^2}\right)\varepsilon}
\end{align}
gives us
\[
\frac{1}{K}\sum_{k=0}^{K-1} \mathbb{E}[\|\nabla  J(\theta^{k})\|^2]
\leq \frac{\mu_F^2\varepsilon}{12M\eta}.
\]
\eqref{equ: NPG 111} and the above inequality yields
\begin{align}
\label{equ: NPG final 2}
\begin{split}
\frac{1}{K}\sum_{k=0}^{K-1}\EE[\|w^k\|^2]&\leq \frac{2}{K}\sum_{k=0}^{K-1}\EE[\|w^k - w^k_{\star}\|^2] + \frac{1}{\mu_F^2}\frac{2}{K}\sum_{k=0}^{K-1} \mathbb{E}[\|\nabla  J(\theta^{k})\|^2]\\
&\leq \frac{2\varepsilon}{12 M\eta}+ \frac{2}{\mu_F^2}\cdot\frac{\mu_F^2\varepsilon}{12M\eta} = \frac{\varepsilon}{3 M\eta}. 
\end{split}
\end{align}
\item Bounding $\frac{1}{K} \mathbb{E}_{s\sim d^{\pi^{\star}}_\rho} \left[\text{KL}\left(\pi^{\star}(\cdot\given s)|| \pi_{\theta^0}(\cdot\given s)\right)\right]$.

Let us also set
\begin{align}
\label{equ: NPG 2}
K \geq \frac{3\mathbb{E}_{s\sim d^{\pi^{\star}}_\rho} \left[\text{KL}\left(\pi^{\star}(\cdot\given s)|| \pi_{\theta^0}(\cdot\given s)\right)\right]}{\eta\varepsilon},
\end{align}
so that
\begin{align}
\label{equ: NPG final 3}
\frac{1}{\eta K} \mathbb{E}_{s\sim d^{\pi^{\star}}_\rho} \left[\text{KL}\left(\pi^{\star}(\cdot\given s)|| \pi_{\theta^0}(\cdot\given s)\right)\right] \leq \frac{\varepsilon}{3}.
\end{align}

\end{itemize}

In summary, we require $K$ to satisfy \eqref{equ: NPG 1} and \eqref{equ: NPG 2}, which leads to 
\[K = \cO\left(\frac{1}{(1-\gamma)^2\varepsilon}\right)\]
By combining \eqref{equ: NPG final 1}, \eqref{equ: NPG final 2}, \eqref{equ: NPG final 3} and \eqref{equ: equ of global convergence}, we can conclude that
\begin{align*}
\begin{split}
J(\pi^{\star})-\frac{1}{K}\sum_{k=0}^{K-1}J(\theta^k)
&\leq \frac{\sqrt{\varepsilon_{\text{bias}}}}{1-\gamma} + \varepsilon.
\end{split}
\end{align*}
Since at each iteration, SGD needs to sample $\cO\left(\frac{1}{(1-\gamma)^4\varepsilon^2}\right)$ trajectories so that \eqref{equ: NPG 111} is satisfied, NPG requires to sample $\cO\left(\frac{1}{(1-\gamma)^6\varepsilon^3}\right)$ trajectories in total.

%
%
%
%
%
%
%
%
%
%

\section{Proof of Theorem \ref{thm: SRVR-PG global convergence}}
\label{app: SRVR-PG global}

Let us take $w^{j+1}_t$ as the update direction of SRVR-PG and apply Proposition \ref{prop: global convergence}. To this end, we need to upper bound 
$\frac{1}{Sm}\sum_{s=0}^{S-1}\sum_{t=0}^{m-1} \|w^{j+1}_t-w^{j+1}_{t,\star}\|$, $\frac{1}{Sm}\sum_{s=0}^{S-1}\sum_{t=0}^{m-1}\|w^{j+1}_t\|^2$, and \\
$\frac{1}{Sm} \mathbb{E}_{s\sim d^{\pi^{\star}}_\rho} \left[\text{KL}\left(\pi^{\star}(\cdot\given s)|| \pi_{\theta^0}(\cdot\given s)\right)\right]$, where $w^{j+1}_{t,\star}=F^{-1}(\theta^{j+1}_t)\nabla  J(\theta^{j+1}_t)$ is the exact NPG update direction at $\theta^{j+1}_t$.

\begin{itemize}
\item Bounding $\frac{1}{Sm}\sum_{s=0}^{S-1}\sum_{t=0}^{m-1} \|w^{j+1}_t-w^{j+1}_{t,\star}\|$. 


Since $w^{j+1}_t = u^{j+1}_t$ and $w^{j+1}_{t,\star}=F^{-1}(\theta^{j+1}_t)\nabla   J(\theta^{j+1}_t)$, we have from Lemmas \ref{lem: nabla J and v} and \ref{lem: smoothness of objective} that
\begin{align*}
&\mathbb{E}[\|w^{j+1}_{t}-w^{j+1}_{t, \star}\|^2]\\
&\leq  2\mathbb{E}[\|u^{j+1}_{t} - \nabla  J(\theta^{j+1}_t)\|^2]+2\mathbb{E}[\|\nabla  J(\theta^{j+1}_t)-F^{-1}(\theta^{j+1}_t)\nabla  J(\theta^{j+1}_t)\|^2]\\
&\leq 2\mathbb{E}[\|u^{j+1}_t - \nabla  J^H(\theta^{j+1}_t)\|^2] + 2(1+\frac{1}{\mu_F})^2\mathbb{E}[\|\nabla  J(\theta^{j+1}_t)\|^2]\\
&\,\,\, + 2 G^2R^2 \left(\frac{H+1}{1-\gamma}+\frac{\gamma}{(1-\gamma)^2}\right)^2\gamma^{2H}\\
& \leq 2\left(\frac{C_{\gamma}}{B}\sum_{l=1}^t \mathbb{E}[\|\theta^{j+1}_{l}-\theta^{j+1}_{l-1}\|^2]+\frac{\sigma^2}{N}\right) + 2(1+\frac{1}{\mu_F})^2\mathbb{E}[\|\nabla  J(\theta^{j+1}_t)\|^2]\\
&\,\,\, + 2 G^2R^2 \left(\frac{H+1}{1-\gamma}+\frac{\gamma}{(1-\gamma)^2}\right)^2\gamma^{2H}\\
& \leq 2\left(\frac{C_{\gamma}}{B}\sum_{t=0}^{m-1} \mathbb{E}[\|\theta^{j+1}_{t+1}-\theta^{j+1}_{t}\|^2]+\frac{\sigma^2}{N}\right)+ 2(1+\frac{1}{\mu_F})^2\mathbb{E}[\|\nabla  J(\theta^{j+1}_t)\|^2]\\
&\,\,\,+ 2 G^2R^2 \left(\frac{H+1}{1-\gamma}+\frac{\gamma}{(1-\gamma)^2}\right)^2\gamma^{2H},
\end{align*}
where we have applied Lemma \ref{lem: smoothness of objective} and Assumption \ref{assump: strong convexity} in the second inequality, and Lemma \ref{lem: nabla J and v} in the third one.

Telescoping this over $s=0,1,..,S-1$, $t=0,1, m-1$ and dividing by $Sm$ gives
\begin{align}
\label{equ: SRVR-PG 111}
\begin{split}
&\frac{1}{Sm}\sum_{s=0}^{S-1}\sum_{t=0}^{m-1}\mathbb{E}[\|w^{j+1}_{t}-w^{j+1}_{t, \star}\|^2]\\
& \leq 2(1+\frac{1}{\mu_F})^2\frac{1}{Sm}\sum_{s=0}^{S-1}\sum_{t=0}^{m-1}\mathbb{E}[\|\nabla  J(\theta^{j+1}_t)\|^2]\\
&\,\,\, +2\left(\frac{C_{\gamma}m}{B}\frac{1}{Sm}\sum_{s=0}^{S-1} \sum_{t=0}^{m-1} \mathbb{E}[\|\theta^{j+1}_{t+1}-\theta^{j+1}_{t}\|^2]+\frac{\sigma^2}{N}\right)\\
&\,\,\, + 2 G^2R^2 \left(\frac{H+1}{1-\gamma}+\frac{\gamma}{(1-\gamma)^2}\right)^2\gamma^{2H}.
\end{split}
\end{align}


%
%
%
%
%
%
%

On the other hand, from Equation (B.14) of \cite{xu2019sample} we know that
\begin{align}
\label{equ: SRVR-PG stationary final}
\begin{split}	
&\left(\frac{2}{\eta^2}-\frac{4L_J}{\eta}-\frac{12 m C_{\gamma}}{B}\right)\frac{1}{Sm}\sum_{s=0}^{S-1}\sum_{t=0}^{m-1}\mathbb{E}[\|\theta^{j+1}_{t+1}-\theta^{j+1}_{t}\|^2] \\
&\,\,\,+ \frac{1}{Sm}\sum_{s=0}^{S-1}\sum_{t=0}^{m-1}\EE[\|\nabla  J^H(\theta^{j+1}_t)\|^2] \\
&\leq \frac{8( J^{H, \star}- J^H(\theta_0))}{\eta Sm}+\frac{6\sigma^2}{N}. 
\end{split}
\end{align}
By the definition of $C_{\gamma}$ in Lemma \ref{lem: nabla J and v}, we have
\begin{align}
\label{equ: SRVR-PG B}
B = \frac{3\eta C_{\gamma} m}{L_J} = \frac{72 \eta R G (2G^2+M)(W+1)\gamma}{L_J(1-\gamma)^5}m
\end{align}
Therefore, \eqref{equ: SRVR-PG stationary final} becomes
\begin{align*}
\left(\frac{2}{\eta^2}-\frac{8 L_J}{\eta}\right)\frac{1}{Sm}\sum_{s=0}^{S-1}\sum_{t=0}^{m-1}\mathbb{E}[\|\theta^{j+1}_{t+1}-\theta^{j+1}_{t}\|^2]&+\frac{1}{Sm}\sum_{s=0}^{S-1}\sum_{t=0}^{m-1}\EE[\|\nabla  J^H(\theta^{j+1}_t)\|^2] \\
&\leq \frac{8( J^{H, \star}- J^H(\theta_0))}{\eta Sm}+\frac{6\sigma^2}{N},
\end{align*}
Since $\eta =\frac{1}{8 L_J}$, we further have
\begin{align}
\label{equ: SRVR-PG stationary final 1}
\begin{split}
\frac{1}{Sm}\sum_{s=0}^{S-1}\sum_{t=0}^{m-1}\mathbb{E}[\|\theta^{j+1}_{t+1}-\theta^{j+1}_{t}\|^2] &\leq \frac{ J^{H, \star}- J^H(\theta_0)}{L_J Sm}+\frac{6\sigma^2}{64 L_J^2 N},\\
\frac{1}{Sm}\sum_{s=0}^{S-1}\sum_{t=0}^{m-1}\EE[\|\nabla  J^H(\theta^{j+1}_t)\|^2]
 & \leq \frac{64 L_J ( J^{H, \star}- J^H(\theta_0))}{Sm}+\frac{6\sigma^2}{N}.
 \end{split}
\end{align}
Putting these inequalities back into \eqref{equ: SRVR-PG 111} yields
\begin{align*}
\frac{1}{Sm}\sum_{s=0}^{S-1}\sum_{t=0}^{m-1}\mathbb{E}[\|w^{j+1}_{t}-w^{j+1}_{t, \star}\|^2] & \leq 2(1+\frac{1}{\mu_F})^2\left(\frac{64 L_J ( J^{H, \star}- J^H(\theta_0))}{Sm}+\frac{6\sigma^2}{N}\right)\\
& +2\left(\frac{8 L_J^2}{3}\left(\frac{ J^{H, \star}- J^H(\theta_0)}{L_J Sm}+\frac{6\sigma^2}{64 L_J^2 N}\right)+\frac{\sigma^2}{N}\right) \\
& +2 G^2R^2 \left(\frac{H+1}{1-\gamma}+\frac{\gamma}{(1-\gamma)^2}\right)^2\gamma^{2H}.	
\end{align*}

Let us set 
\begin{align}
\label{equ: SRVR-PG 1}
\begin{split}
\frac{1}{3}\left(\frac{\varepsilon}{3G}\right)^2 &\geq 2 G^2R^2 \left(\frac{H+1}{1-\gamma}+\frac{\gamma}{(1-\gamma)^2}\right)^2\gamma^{2H},\\
N &\geq \frac{\left(12(1+\frac{1}{\mu_F})^2+2.5\right)\sigma^2}{\frac{1}{3}\left(\frac{\varepsilon}{3G}\right)^2},\\
Sm &\geq  \frac{\left(128(1+\frac{1}{\mu_F})^2+\frac{16}{3}\right)L_J( J^{H, \star}- J^H(\theta_0))}{\frac{1}{3}\left(\frac{\varepsilon}{3G}\right)^2},	
\end{split}
\end{align}
so that 
\begin{align*}
\frac{1}{Sm}\sum_{s=0}^{S-1}\sum_{t=0}^{m-1}\mathbb{E}[\|w^{j+1}_{t}-w^{j+1}_{t, \star}\|^2]\leq \left(\frac{1}{G}\frac{\varepsilon}{3}\right)^2
\end{align*}

By Jensen's inequality, we further have
\begin{align}
\label{equ: SRVR-PG final 1}.
\begin{split}
\left(\frac{1}{Sm}\sum_{s=0}^{S-1}\sum_{t=0}^{m-1}\mathbb{E}[\|w^{j+1}_{t}-w^{j+1}_{t, \star}\|]\right)^2&\leq \frac{1}{Sm}\sum_{s=0}^{S-1}\sum_{t=0}^{m-1}\left(\mathbb{E}[\|w^{j+1}_{t}-w^{j+1}_{t, \star}\|]\right)^2\\
&\leq \frac{1}{Sm}\sum_{s=0}^{S-1}\sum_{t=0}^{m-1}\mathbb{E}[\|w^{j+1}_{t}-w^{j+1}_{t, \star}\|^2]\\
& \leq \left(\frac{1}{G}\frac{\varepsilon}{3}\right)^2.
\end{split}
\end{align}
where we have also applied $\left(\mathbb{E}[\|w^{j+1}_{t}-w^{j+1}_{t, \star}\|]\right)^2\leq \mathbb{E}[\|w^{j+1}_{t}-w^{j+1}_{t, \star}\|^2]$.

\item Bounding $\frac{1}{Sm}\sum_{s=0}^{S-1}\sum_{t=0}^{m-1}\|w^{j+1}_t\|^2$

Since $\EE [w^{j+1}_t]=\EE[u^{j+1}_t]=\nabla  J^H(\theta^{j+1}_t)$, by Lemma \ref{lem: nabla J and v} we have
\begin{align*}
&\frac{1}{Sm}\sum_{s=0}^{S-1}\sum_{t=0}^{m-1}\EE[\|w^{j+1}_t\|^2] \\
&= \frac{1}{Sm}\sum_{s=0}^{S-1}\sum_{t=0}^{m-1}\EE[\|u^{j+1}_t-\nabla  J^H(\theta^{j+1}_t)\|^2]+\frac{1}{Sm}\sum_{s=0}^{S-1}\sum_{t=0}^{m-1}\EE[\|\nabla  J^H(\theta^{j+1}_t)\|^2]\\
&\leq \frac{1}{Sm}\sum_{s=0}^{S-1}\sum_{t=0}^{m-1}\left(\frac{C_{\gamma}}{B}\sum_{l=1}^t \mathbb{E}[\|\theta^{j+1}_{l}-\theta^{j+1}_{l-1}\|^2]+\frac{\sigma^2}{N}\right) \\
&\,\,\, + \frac{1}{Sm}\sum_{s=0}^{S-1}\sum_{t=0}^{m-1}\EE[\|\nabla  J^H(\theta^{j+1}_t)\|^2]\\
&\leq \frac{C_{\gamma} m}{B}\cdot\frac{1}{Sm}\sum_{s=0}^{S-1}\sum_{t=0}^{m-1} \mathbb{E}[\|\theta^{j+1}_{t+1}-\theta^{j+1}_{t}\|^2]+\frac{\sigma^2}{N} + \frac{1}{Sm}\sum_{s=0}^{S-1}\sum_{t=0}^{m-1}\EE[\|\nabla  J^H(\theta^{j+1}_t)\|^2].
\end{align*}
By setting $\eta = \frac{1}{8L_J}$ and applying \eqref{equ: SRVR-PG B} and \eqref{equ: SRVR-PG stationary final 1}, we further have
\begin{align*}
&\frac{1}{Sm}\sum_{s=0}^{S-1}\sum_{t=0}^{m-1}\EE[\|w^{j+1}_t\|^2]\\
&\leq \frac{8L_J^2}{3}\cdot\left(\frac{ J^{H, \star}- J^H(\theta_0)}{L_J Sm}+\frac{6\sigma^2}{64 L_J^2 N}\right)+\frac{\sigma^2}{N} + \frac{64 L_J ( J^{H, \star}- J^H(\theta_0))}{Sm}+\frac{6\sigma^2}{N}
\end{align*}
Therefore, we can set 
\begin{align}
\label{equ: SRVR-PG 2}
\begin{split}
N &\geq \frac{174 M\sigma^2}{32L_J \varepsilon},\\
Sm &\geq  \frac{50 M ( J^{H, \star}- J^H(\theta_0))}{\varepsilon},	
\end{split}
\end{align}
so that 
\begin{align}
\label{equ: SRVR-PG final 2}
\frac{1}{Sm}\sum_{s=0}^{S-1}\sum_{t=0}^{m-1}\mathbb{E}\|w^{j+1}_{t}\|^2\leq \frac{\varepsilon}{3M\eta}.
\end{align}

\item Bounding $\frac{1}{Sm} \mathbb{E}_{s\sim d^{\pi^{\star}}_\rho} \left[\text{KL}\left(\pi^{\star}(\cdot\given s)|| \pi_{\theta^0}(\cdot\given s)\right)\right]$.

Let us set 
\begin{align}
\label{equ: SRVR-PG 3}
Sm \geq \frac{3 \mathbb{E}_{s\sim d^{\pi^{\star}}_\rho} \left[\text{KL}\left(\pi^{\star}(\cdot\given s)|| \pi_{\theta^0}(\cdot\given s)\right)\right]}{\eta \varepsilon}
\end{align}
so that
\begin{align}
\label{equ: SRVR-PG final 3}
\frac{1}{\eta Sm} \mathbb{E}_{s\sim d^{\pi^{\star}}_\rho} \left[\text{KL}\left(\pi^{\star}(\cdot\given s)|| \pi_{\theta^0}(\cdot\given s)\right)\right]\leq \frac{\varepsilon}{3}.
\end{align}

%
%
%
%
\end{itemize}

By combining \eqref{equ: SRVR-PG final 1}, \eqref{equ: SRVR-PG final 2}, \eqref{equ: SRVR-PG final 3} and \eqref{equ: equ of global convergence}, we can conclude that
\begin{align*}
\begin{split}
J(\pi^{\star})-\frac{1}{K}\sum_{k=0}^{K-1}J(\theta^k)
&\leq \frac{\sqrt{\varepsilon_{\text{bias}}}}{1-\gamma} + \varepsilon.
\end{split}
\end{align*}
To achieve this, we require $Sm$ and $N$ to satisfy \eqref{equ: SRVR-PG 1}, \eqref{equ: SRVR-PG 2}, and \eqref{equ: SRVR-PG 3}, which leads to 
\[
Sm = \cO\left(\frac{1}{(1-\gamma)^2\varepsilon^2}\right), \quad\quad N = \cO\left(\frac{\sigma^2}{\varepsilon^2}\right), \qquad H = \cO\left(\log(\frac{1}{(1-\gamma)\varepsilon})\right).
\]
By \eqref{equ: SRVR-PG B}, we know that $B = \cO(W (1-\gamma)^{-1}m)$. 

Therefore, by taking $S=\cO\left(\frac{1}{(1-\gamma)^{2.5}\varepsilon}\right)$ and $m=\cO\left(\frac{1}{(1-\gamma)^{-0.5}\varepsilon}\right)$, the sample complexity of SRVR-PG is
\begin{align*}
S(N+mB)
& = \cO\left(\frac{\sigma^2}{(1-\gamma)^{2.5}\varepsilon^{3}}+\frac{W}{(1-\gamma)^{2.5}\varepsilon^{3}}\right)\\
& = \cO\left(\frac{W+\sigma^2}{(1-\gamma)^{2.5}\varepsilon^{3}}\right).
\end{align*}

\section{Proof of Theorem \ref{thm: SRVR-NPG global convergence}}
\label{app: SRVR-NPG global}

Let us take $w^{j+1}_t$ as the update direction of SRVR-NPG and apply Proposition \ref{prop: global convergence}. To this end, we need to upper bound 
$\frac{1}{Sm}\sum_{s=0}^{S-1}\sum_{t=0}^{m-1} \|w^{j+1}_t-w^{j+1}_{t,\star}\|$, $\frac{1}{Sm}\sum_{s=0}^{S-1}\sum_{t=0}^{m-1}\|w^{j+1}_t\|^2$, and \\
$\frac{1}{Sm} \mathbb{E}_{s\sim d^{\pi^{\star}}_\rho} \left[\text{KL}\left(\pi^{\star}(\cdot\given s)|| \pi_{\theta^0}(\cdot\given s)\right)\right]$, where $w^{j+1}_{t,\star}=F^{-1}(\theta^{j+1}_t)\nabla J(\theta^{j+1}_t)$ is the exact NPG update direction at $\theta^{j+1}_t$.

Let us take $\eta = \frac{\mu_F}{16 L_J}$ and apply SGD as in Procedure \ref{alg: SGD for SRVR-NPG subproblem} to obtain a $w^{j+1}_t$ that satisfies
\begin{align}
\label{equ: SRVR-NPG 111}
\EE[\|w^{j+1}_t-F^{-1}(\theta^{j+1}_t)u^{j+1}_t\|^2]\leq \min\Big\{\frac{\frac{1}{4}\left(\frac{1}{G}\frac{\varepsilon}{3}\right)^2}{2 + \frac{G^2\mu_F+G^4}{\mu_F^2}\cdot \frac{\mu_F}{4G^2+\mu_F}}, \frac{64\eta^2L_J^2}{\mu_F(\mu_F+G^2)}\cdot\frac{\varepsilon}{9M\eta}\Big\}.
\end{align}
In order to apply Proposition \ref{prop: SGD convergence SRVR-NPG}, let assume the following so that its assumptions are satisfied:
\begin{align}
\label{equ: assumptions of prop}
\begin{split}
\frac{\sigma^2}{N}&\leq  \left(\frac{GR}{(1-\gamma)^2}\right)^2,\\
\min\Big\{\frac{\frac{1}{4}\left(\frac{1-\gamma}{G}\frac{\varepsilon}{3}\right)^2}{2 + \frac{G^2\mu_F+G^4}{\mu_F^2}\cdot \frac{\mu_F}{4G^2+\mu_F}}, \frac{64\eta^2L_J^2}{\mu_F(\mu_F+G^2)}\cdot\frac{(1-\gamma)\varepsilon}{9M\eta}\Big\} &\leq  \frac{2}{\mu_F^2}\left(\frac{GR}{(1-\gamma)^2}\right)^2\\
\frac{C_\gamma m}{B}2\eta^2&\leq \frac{1}{3}\mu_F^2. 
\end{split}
\end{align}
At the end of this proof, we will see that these assumptions are indeed satisfied for small $\varepsilon$.

From Proposition \ref{prop: SGD convergence SRVR-NPG}, we know that this requires sampling $\cO\left(\frac{1}{(1-\gamma)^4\varepsilon^2}\right)$ trajectories at each iteration. 

\begin{itemize}
\item Bounding $\frac{1}{Sm}\sum_{s=0}^{S-1}\sum_{t=0}^{m-1} \|w^{j+1}_t-w^{j+1}_{t,\star}\|$.


First of all, we have 
\begin{align*}
&\mathbb{E}[\|w^{j+1}_{t}-w^{j+1}_{t, \star}\|^2]\\
&\leq  2\mathbb{E}[\|w^{j+1}_{t} - F^{-1}(\theta^{j+1}_t)u^{j+1}_t\|^2]+2\mathbb{E}[\|F^{-1}(\theta^{j+1}_t)u^{j+1}_t-F^{-1}(\theta^{j+1}_t)\nabla  J(\theta^{j+1}_t)\|^2]\\
&\leq 2\mathbb{E}[\|w^{j+1}_t - F^{-1}(\theta^{j+1}_t)u^{j+1}_t\|^2] + 2\frac{1}{\mu_F^2}\mathbb{E}[\|u^{j+1}_t - \nabla  J(\theta^{j+1}_t)\|^2]\\
& \leq 2\EE[\|w^{j+1}_t-F^{-1}(\theta^{j+1}_t)u^{j+1}_t\|^2] + 4\frac{1}{\mu_F^2}\left(\frac{C_{\gamma}}{B}\sum_{l=1}^t \mathbb{E}[\|\theta^{j+1}_{l}-\theta^{j+1}_{l-1}\|^2]+\frac{\sigma^2}{N}\right)\\
&\,\,\, + 4 G^2R^2 \left(\frac{H+1}{1-\gamma}+\frac{\gamma}{(1-\gamma)^2}\right)^2\gamma^{2H}\\
& \leq 2\EE[\|w^{j+1}_t-F^{-1}(\theta^{j+1}_t)u^{j+1}_t\|^2] + 4\frac{1}{\mu_F^2}\left(\frac{C_{\gamma}}{B}\sum_{t=0}^{m-1} \mathbb{E}[\|\theta^{j+1}_{t+1}-\theta^{j+1}_{t}\|^2]+\frac{\sigma^2}{N}\right)\\
&\,\,\, + 4 G^2R^2 \left(\frac{H+1}{1-\gamma}+\frac{\gamma}{(1-\gamma)^2}\right)^2\gamma^{2H},
\end{align*}
where we have applied Assumption \ref{assump: strong convexity} in the second inequality, and Lemmas \ref{lem: nabla J and v} and \ref{lem: smoothness of objective} in the third one.

Telescoping this over $s=0,1,..,S-1$, $t=0,1, m-1$ and dividing by $Sm$ gives 
\begin{align}
\label{equ: SRVR-NPG 1111}
\begin{split}
&\frac{1}{Sm}\sum_{s=0}^{S-1}\sum_{t=0}^{m-1}\mathbb{E}[\|w^{j+1}_{t}-w^{j+1}_{t, \star}\|^2]\\
& \leq 2\frac{\frac{1}{4}\left(\frac{1}{G}\frac{\varepsilon}{3}\right)^2}{2 + \frac{G^2\mu_F+G^4}{\mu_F^2}\cdot \frac{\mu_F}{4G^2+\mu_F}} + 4\frac{1}{\mu_F^2}\left(\frac{C_{\gamma}m}{B}\frac{1}{Sm}\sum_{s=0}^{S-1} \sum_{t=0}^{m-1} \mathbb{E}[\|\theta^{j+1}_{t+1}-\theta^{j+1}_{t}\|^2]+\frac{\sigma^2}{N}\right)\\
&\,\,\, + 4 G^2R^2 \left(\frac{H+1}{1-\gamma}+\frac{\gamma}{(1-\gamma)^2}\right)^2\gamma^{2H}.
\end{split}
\end{align}

On the other hand, from \eqref{equ: SRVR-NPG final} we know that
\begin{align}
\label{equ: SRVR-NPG stationary final}
\begin{split}
&\frac{8G^2}{\eta}\left(\frac{\mu_F}{8\eta} - \frac{L_J}{2}-\left(\frac{\eta}{\mu_F}+\frac{\eta}{4G^2}\right)\frac{C_{\gamma}m}{B}\right)\frac{1}{Sm}\sum_{s=0}^{S-1}\sum_{t=0}^{m-1}\EE\|\theta^{j+1}_{t+1}-\theta^{j+1}_t\|^2\\
&\,\,\, +\frac{1}{Sm} \sum_{s=0}^{S-1}\sum_{t=0}^{m-1} \EE\|\nabla  J^H(\theta^{j+1}_t)\|^2\\ 
&\leq \frac{8G^2}{\eta}\frac{ J^{H, \star}- J^H(\theta_0)}{Sm}\\
&\,\,\,  + \left(\frac{8G^2}{\mu_F}+2\right)\frac{\sigma^2}{N}+\left(\frac{8G^2\mu_F}{4}+\frac{8G^4}{4}\right)\frac{1}{Sm}\sum_{s=0}^{S-1}\sum_{t=0}^{m-1}\EE\|F^{-1}(\theta^{j+1}_t)u^{j+1}_t - w^{j+1}_{t}\|^2.
\end{split}
\end{align}
Let us set
\begin{align}
\label{equ: SRVR-NPG B}
B \geq \left(\frac{\eta}{\mu_F}+\frac{\eta}{4G^2}\right)\frac{2 C_{\gamma} m}{L_J} = \left(\frac{\eta}{\mu_F}+\frac{\eta}{4G^2}\right)\cdot \frac{48 R G^2 (2G^2+M)(W+1)\gamma}{L_J(1-\gamma)^5}m.
\end{align}
Since $\eta =\frac{\mu_F}{16 L_J}$, \eqref{equ: SRVR-NPG stationary final} becomes
\begin{align*}
&\frac{8G^2 L_J}{\eta}\frac{1}{Sm}\sum_{s=0}^{S-1}\sum_{t=0}^{m-1}\EE\|\theta^{j+1}_{t+1}-\theta^{j+1}_t\|^2+\frac{1}{Sm} \sum_{s=0}^{S-1}\sum_{t=0}^{m-1} \EE\|\nabla  J^H(\theta^{j+1}_t)\|^2\\ 
&\leq \frac{8G^2}{\eta}\frac{ J^{H, \star}- J^H(\theta_0)}{Sm}+ \left(\frac{8G^2}{\mu_F}+2\right)\frac{\sigma^2}{N}\\
&\,\,\, +\left(\frac{8G^2\mu_F}{4}+\frac{8G^4}{4}\right)\frac{1}{Sm}\sum_{s=0}^{S-1}\sum_{t=0}^{m-1}\EE\|F^{-1}(\theta^{j+1}_t)u^{j+1}_t - w^{j+1}_{t}\|^2,
\end{align*}
from which we have
\begin{align}
\label{equ: SRVR-NPG stationary final 1}
\begin{split}
&\frac{1}{Sm}\sum_{s=0}^{S-1}\sum_{t=0}^{m-1}\mathbb{E}[\|\theta^{j+1}_{t+1}-\theta^{j+1}_{t}\|^2] \\
&\leq \frac{ J^{H, \star}- J^H(\theta_0)}{L_J Sm}+\left(\frac{8G^2}{\mu_F}+2\right)\frac{\mu_F}{128 G^2 L_J^2}\frac{\sigma^2}{ N}\\
&\,\,\, +\left(\frac{8G^2\mu_F}{4}+\frac{8G^4}{4}\right)\frac{\mu_F}{128 G^2 L_J^2}\frac{1}{Sm}\sum_{s=0}^{S-1}\sum_{t=0}^{m-1}\EE\|F^{-1}(\theta^{j+1}_t)u^{j+1}_t - w^{j+1}_{t}\|^2.
 \end{split}
\end{align}
Putting these inequalities back into \eqref{equ: SRVR-NPG 1111} and applying \eqref{equ: SRVR-NPG 111} yields
\begin{align*}
&\frac{1}{Sm}\sum_{s=0}^{S-1}\sum_{t=0}^{m-1}\mathbb{E}[\|w^{j+1}_{t}-w^{j+1}_{t, \star}\|^2] \\
&\leq 2\frac{\frac{1}{4}\left(\frac{1}{G}\frac{\varepsilon}{3}\right)^2}{2 + \frac{G^2\mu_F+G^4}{\mu_F^2}\cdot \frac{\mu_F}{4G^2+\mu_F}} \\
& +4\frac{1}{\mu_F^2}\left(\frac{C_{\gamma}m}{B}\frac{1}{Sm}\sum_{s=0}^{S-1} \sum_{t=0}^{m-1} \mathbb{E}[\|\theta^{j+1}_{t+1}-\theta^{j+1}_{t}\|^2]+\frac{\sigma^2}{N}\right)\\
&\,\,\, + 4 G^2R^2 \left(\frac{H+1}{1-\gamma}+\frac{\gamma}{(1-\gamma)^2}\right)^2\gamma^{2H}\\
&\leq 2 \frac{\frac{1}{4}\left(\frac{1}{G}\frac{\varepsilon}{3}\right)^2}{2 + \frac{G^2\mu_F+G^4}{\mu_F^2}\cdot \frac{\mu_F}{4G^2+\mu_F}} \\
&\,\,\, + 2\frac{1}{\mu_F^2}\frac{C_{\gamma}m}{B} \Bigg(\frac{J^{\star}-J(\theta_0)}{L_J Sm}+\left(\frac{8G^2}{\mu_F}+2\right)\frac{\mu_F}{128 G^2 L_J^2}\frac{\sigma^2}{N}\\
&\,\,\,\,\,\,\qquad  +\left(G^2\mu_F+ G^4\right)\frac{\mu_F}{64 G^2 L_J^2}\frac{\frac{1}{3}\left(\frac{1-\gamma}{G}\frac{\varepsilon}{3}\right)^2}{1 + \frac{G^2\mu_F+G^4}{\mu_F^2}\cdot \frac{\mu_F}{4G^2+\mu_F}}\Bigg)\\
&\,\,\, + 2\frac{1}{\mu_F^2}\frac{\sigma^2}{N} + 4 G^2R^2 \left(\frac{H+1}{1-\gamma}+\frac{\gamma}{(1-\gamma)^2}\right)^2\gamma^{2H}.
\end{align*}
From \eqref{equ: SRVR-NPG B} we know that
\[
\frac{C_{\gamma}m}{B}\leq \frac{32L_J^2}{4+\frac{\mu_F}{G^2}},
\]
which gives us
\begin{align*}
&\frac{1}{Sm}\sum_{s=0}^{S-1}\sum_{t=0}^{m-1}\mathbb{E}[\|w^{j+1}_{t}-w^{j+1}_{t, \star}\|^2] \\
&\leq \left(2 + \frac{G^2\mu_F+G^4}{\mu_F^2}\cdot \frac{\mu_F}{4G^2+\mu_F}\right)\frac{\frac{1}{4}\left(\frac{1}{G}\frac{\varepsilon}{3}\right)^2}{2 + \frac{G^2\mu_F+G^4}{\mu_F^2}\cdot \frac{\mu_F}{4G^2+\mu_F}} \\
&\,\,\,+ 2\frac{1}{\mu_F^2} \frac{C_{\gamma}m}{B} \frac{ J^{H, \star}- J^H(\theta_0)}{L_J Sm}\\
&\,\,\, + 2\frac{1}{\mu_F^2}\left(1 + (\frac{8G^2}{\mu_F}+2)\frac{\mu_F}{4(4G^2+\mu_F)}\right) \frac{\sigma^2}{N}\\
&\,\,\, + 4 G^2R^2 \left(\frac{H+1}{1-\gamma}+\frac{\gamma}{(1-\gamma)^2}\right)^2\gamma^{2H}\\
&= \frac{1}{4}\left(\frac{1}{G}\frac{\varepsilon}{3}\right)^2 + 2\frac{1}{\mu_F^2} \frac{C_{\gamma}m}{B} \frac{ J^{H, \star}- J^H(\theta_0)}{L_J Sm}\\
&\,\,\, + 3\frac{1}{\mu_F^2}\frac{\sigma^2}{N} + 4 G^2R^2 \left(\frac{H+1}{1-\gamma}+\frac{\gamma}{(1-\gamma)^2}\right)^2\gamma^{2H},
\end{align*}
where we have applied \eqref{equ: SRVR-NPG 111} in the first equality.

Let us set 
\begin{align}
\label{equ: SRVR-NPG 1}
\begin{split}
N &\geq \frac{108 G^2 \sigma^2}{\mu_F^2\varepsilon^2},\\
B &\geq \frac{72 C_{\gamma}m}{\mu_F L_J^2\varepsilon},\\
Sm &\geq  \frac{L_J( J^{H, \star}- J^H(\theta_0))G^2}{\mu_F\varepsilon},	\\
\frac{1}{4}\left(\frac{\varepsilon}{3G}\right)^2 &\geq 4 G^2R^2 \left(\frac{H+1}{1-\gamma}+\frac{\gamma}{(1-\gamma)^2}\right)^2\gamma^{2H},
\end{split}
\end{align}
so that 
\begin{align*}
\frac{1}{Sm}\sum_{s=0}^{S-1}\sum_{t=0}^{m-1}\mathbb{E}[\|w^{j+1}_{t}-w^{j+1}_{t, \star}\|^2]\leq \left(\frac{1}{G}\frac{\varepsilon}{3}\right)^2
\end{align*}

By Jensen's inequality, we further have
\begin{align}
\label{equ: SRVR-NPG final 1}.
\begin{split}
\left(\frac{1}{Sm}\sum_{s=0}^{S-1}\sum_{t=0}^{m-1}\mathbb{E}[\|w^{j+1}_{t}-w^{j+1}_{t, \star}\|]\right)^2&\leq \frac{1}{Sm}\sum_{s=0}^{S-1}\sum_{t=0}^{m-1}\left(\mathbb{E}[\|w^{j+1}_{t}-w^{j+1}_{t, \star}\|]\right)^2\\
&\leq \frac{1}{Sm}\sum_{s=0}^{S-1}\sum_{t=0}^{m-1}\mathbb{E}[\|w^{j+1}_{t}-w^{j+1}_{t, \star}\|^2]\\
& \leq \left(\frac{1}{G}\frac{\varepsilon}{3}\right)^2.
\end{split}
\end{align}
where we have also applied $\left(\mathbb{E}[\|w^{j+1}_{t}-w^{j+1}_{t, \star}\|]\right)^2\leq \mathbb{E}[\|w^{j+1}_{t}-w^{j+1}_{t, \star}\|^2]$.

\item Bounding $\frac{1}{Sm}\sum_{s=0}^{S-1}\sum_{t=0}^{m-1}\|w^{j+1}_t\|^2$.

By \eqref{equ: SRVR-NPG stationary final 1} we have
\begin{align*}
&\frac{1}{Sm}\sum_{s=0}^{S-1}\sum_{t=0}^{m-1}\EE[\|w^{j+1}_t\|^2]\\
&= \frac{1}{Sm}\frac{1}{\eta^2}\sum_{s=0}^{S-1}\sum_{t=0}^{m-1}\EE[\|\theta^{j+1}_{t+1}-\theta^{j+1}_t\|^2]\\
&\leq \frac{J^{\star}-J(\theta_0)}{L_J\eta^2 Sm}+\left(\frac{8G^2}{\mu_F}+2\right)\frac{\mu_F}{128 G^2 \eta^2 L_J^2}\frac{\sigma^2}{N}\\
&\,\,\, +\left(\frac{8G^2\mu_F}{4}+\frac{8G^4}{4}\right)\frac{\mu_F}{128\eta^2  G^2 L_J^2}\frac{1}{Sm}\sum_{s=0}^{S-1}\sum_{t=0}^{m-1}\EE\|F^{-1}(\theta^{j+1}_t)u^{j+1}_t - w^{j+1}_{t}\|^2 \\
&\leq \frac{J^{\star}-J(\theta_0)}{L_J\eta^2 Sm}+\left(\frac{8G^2}{\mu_F}+2\right)\frac{\mu_F}{128 G^2 \eta^2 L_J^2}\frac{\sigma^2}{N} + \frac{\varepsilon}{9 M\eta},
\end{align*}
where we have applied \eqref{equ: SRVR-NPG stationary final 1} in the first inequality, and \eqref{equ: SRVR-NPG 111} in the last step.

We can set

\begin{align}
\label{equ: SRVR-NPG 2}
\begin{split}
N &\geq \frac{9M\mu_F(\frac{8G^2}{\mu_F}+2)\sigma^2}{128G^2\eta L^2_J\varepsilon},\\
Sm &\geq  \frac{9M(J^{\star}-J(\theta_0))}{L_J \eta \varepsilon},	
\end{split}
\end{align}
so that 
\begin{align}
\label{equ: SRVR-NPG final 2}
\frac{1}{Sm}\sum_{s=0}^{S-1}\sum_{t=0}^{m-1}\mathbb{E}\|w^{j+1}_{t}\|^2\leq \frac{\varepsilon}{3M\eta}.
\end{align}

\item Bounding $\frac{1}{Sm} \mathbb{E}_{s\sim d^{\pi^{\star}}_\rho} \left[\text{KL}\left(\pi^{\star}(\cdot\given s)|| \pi_{\theta^0}(\cdot\given s)\right)\right]$.

Let us set 
\begin{align}
\label{equ: SRVR-NPG 3}
Sm \geq \frac{3 \mathbb{E}_{s\sim d^{\pi^{\star}}_\rho} \left[\text{KL}\left(\pi^{\star}(\cdot\given s)|| \pi_{\theta^0}(\cdot\given s)\right)\right]}{\eta \varepsilon}
\end{align}
so that
\begin{align}
\label{equ: SRVR-NPG final 3}
\frac{1}{\eta Sm} \mathbb{E}_{s\sim d^{\pi^{\star}}_\rho} \left[\text{KL}\left(\pi^{\star}(\cdot\given s)|| \pi_{\theta^0}(\cdot\given s)\right)\right]\leq \frac{\varepsilon}{3}.
\end{align}

\end{itemize}

By combining \eqref{equ: SRVR-NPG final 1}, \eqref{equ: SRVR-NPG final 2}, \eqref{equ: SRVR-NPG final 3} and \eqref{equ: equ of global convergence}, we can conclude that
\begin{align*}
\begin{split}
J(\pi^{\star})-\frac{1}{K}\sum_{k=0}^{K-1}J(\theta^k)
&\leq \frac{\sqrt{\varepsilon_{\text{bias}}}}{1-\gamma} + \varepsilon.
\end{split}
\end{align*}
To achieve this, we require $Sm$, $B$, and $N$ to satisfy \eqref{equ: SRVR-NPG B}, \eqref{equ: SRVR-NPG 1}, \eqref{equ: SRVR-NPG 2}, and \eqref{equ: SRVR-NPG 3}, which leads to 
\begin{align*}
& Sm = \cO\left(\frac{1}{(1-\gamma)^2\varepsilon}\right), \quad\quad N = \cO\left(\frac{\sigma^2}{\varepsilon^2}\right), \quad\quad \\
& B = \cO\left(\frac{W}{(1-\gamma)\varepsilon}m\right), \qquad H =\cO\left(\log\left(\frac{1}{(1-\gamma)\varepsilon}\right)\right).
\end{align*}

By Proposition \ref{prop: SGD convergence SRVR-NPG}, we know that in order to achieve \eqref{equ: SRVR-NPG 111}, SGD requires sampling $\cO\left(\frac{1}{(1-\gamma)^{4}\varepsilon^{2}}\right)$ trajectories per iteration.

Therefore, by taking $S=\cO\left(\frac{1}{(1-\gamma)^{2.5}\varepsilon^{0.5}}\right)$ and $m=\cO\left(\frac{(1-\gamma)^{0.5}}{{\varepsilon}^{0.5}}\right)$, the amount of trajectories required by SRVR-NPG is
\begin{align*}
&S\bigg(N+mB + (1+m)\cO\left(\frac{1}{(1-\gamma)^{4}\varepsilon^{2}}\right)\bigg)\\
& = \cO\left(\frac{\sigma^2}{(1-\gamma)^{2.5}\varepsilon^{2.5}}+\frac{W}{(1-\gamma)^{2.5}\varepsilon^{2.5}}+ \frac{1}{(1-\gamma)^{6}\varepsilon^{3}}\right).
\end{align*}
It is straightforward to verify that the requirements listed in \eqref{equ: assumptions of prop} are also satisfied as long as $\varepsilon$ is small enough. 

\section{Implementation Details}
\label{app: implementation details}

In this section, we provide additional details on the implementation of PG, NPG, SRVR-PG and SRVR-NPG.
\begin{enumerate}
	\item For NPG, we use the default implementation provided by rllab\footnotemark[1], which actually implements the trust region policy optimization(TRPO) algorithm \cite{schulman2015trust}. For cartplole, we sample 200 trajectories at each iteration to solve the subproblem of TRPO. For mountain car, we sample 120 trajectories at each iteration.
	\item We found that the naive implementation of PG and SRVR-PG  typically do not work for our tests. For example, PG and SRVR-PG often give an average reward around $-90$ for the mountain-car test, despite of our best efforts.
	\item As in \cite{papini2018stochastic} and \cite{xu2019sample}, we found that it is necessary to apply Adagrad \cite{duchi2011adaptive} or Adam \cite{kingma2014adam} type of averaging to improve their performances. 
	\item In our experiments, we apply Adagrad type of averaging for PG and SRVR-PG, which results in much better performances. As for SRVR-NPG, we apply Adam type of averaging, which gives an approximation of the Fisher information matrix at each iteration (see section 11.2 of \cite{martens2014new}). We leave the implementation of a better approximation of the Fisher information matrix to the future work.
\end{enumerate}

\footnotetext[1]{\url{https://github.com/rll/rllab}}

\end{document}